\documentclass[10pt,onecolumn,journal, draftclsnofoot]{IEEEtran}

\usepackage[utf8]{inputenc}
\usepackage{amsmath, amsfonts, amssymb, amsthm}
\usepackage{graphicx, epstopdf}
\usepackage{tabulary, booktabs} % for tables and \toprule and \bottomrule
\usepackage{algorithm,algorithmic}
\usepackage{bm} % bold greek letters for math

\usepackage{subcaption} %[labelformat=simple]

\usepackage{tabstackengine} % for equispaces matrix entries
\stackMath

\usepackage{personal} % style file 

\usepackage{chngcntr} % changing counter number in pagraphs
\counterwithin*{paragraph}{section} % paragraph counter depends on section

\allowdisplaybreaks %allow page breaks during equations

\begin{document}

\title{Stochastic Contextual Bandits with Known Reward Functions}
\author{Pranav Sakulkar and Bhaskar Krishnamachari \\
Ming Hsieh Department of Electrical Engineering \\ 
Viterbi School of Engineering \\
University of Southern California, Los Angeles, CA, USA \\
{\tt \{sakulkar, bkrishna\}@usc.edu } }
\date{\today}

\maketitle

\begin{abstract}
Many sequential decision-making problems in communication networks such as power allocation in energy harvesting communications, mobile computational offloading, and dynamic channel selection can be modeled as contextual bandit problems which are natural extensions of the well-known multi-armed bandit problem. In these problems, each resource allocation or selection decision can make use of available side-information such as harvested power, specifications of the jobs to be offloaded, or the number of packets to be transmitted. In contextual bandit problems, at each step of a sequence of trials, an agent observes the side information or context, pulls one arm and receives the reward for that arm. We consider the stochastic formulation where the context-reward tuples are independently drawn from an unknown distribution in each trial. The goal is to design strategies for minimizing the expected cumulative regret or reward loss compared to a distribution-aware genie. We analyze a setting where the reward is a known non-linear function of the context and the chosen arm's current state which is the case with many networking applications. This knowledge of the reward function enables us to exploit the obtained reward information to learn about rewards for other possible contexts. We first consider the case of discrete and finite context-spaces and propose DCB($\e$), an algorithm that yields regret which scales logarithmically in time and linearly in the number of arms that are not optimal for any context. This is in contrast with existing algorithms where the regret scales linearly in the total number of arms. Also, the storage requirements of DCB($\e$) do not grow with time. DCB($\e$) is an extension of the UCB1 policy for the standard multi-armed bandits to contextual bandits using sophisticated proof techniques for regret analysis. We then study continuous context-spaces with Lipschitz reward functions and propose CCB($\e, \d$), an algorithm that uses DCB($\e$) as a subroutine. CCB($\e, \d$) reveals a novel regret-storage trade-off that is parametrized by $\d$. Tuning $\d$ to the time horizon allows us to obtain sub-linear regret bounds, while requiring sub-linear storage. Joint learning for all the contexts results in regret bounds that are unachievable by any existing contextual bandit algorithm for continuous context-spaces. Similar performance bounds are also shown to hold for unknown horizon case by employing a doubling trick. 
\end{abstract}


\section{Introduction}
Many problems in networking such as dynamic spectrum allocation in cognitive radio networks involve sequential decision making under the face of uncertainty. The multi-armed bandit problem (MAB, see \cite{lai1985, agrawal1995, auer2002, vakili2013}), a fundamental online learning framework, has been previously applied in networks for multi-user channel allocation \cite{gai2012comb} and distributed opportunistic spectrum access \cite{gai2014dist}. These problems also appear widely in online advertising and clinical trials. In the standard MAB version, an agent is presented with a sequence of trials where in each trial it has to choose an arm from a set of $K$ arms, each providing stochastic rewards over time with unknown distributions. The agent receives a payoff or reward based on its action in each trial. The goal is to maximize the total expected reward over time. On one hand, different arms need to be explored often enough to learn their rewards and on the other hand, prior observations need to be exploited to maximize the immediate reward. MAB problems, therefore, capture the fundamental trade-off between exploration and exploitation that appear widely in these decision making problems.

The contextual bandit problem considered in this paper is a natural extension of the basic MAB problem where the agent can see a context or some hint about the rewards in the current trial that can be used to make a better decision. This makes contextual bandits more general and more suitable to various practical applications such as ad placements \cite{langford2008epoch, lu2010lipschitz} and personalized news recommendations \cite{li2010}, since settings with no context information are rare in practice. As we show in this work, contextual bandits are also applicable to many problems arising in wireless communications and networking.

In this paper, we consider contextual bandit problems where the relationship between the contexts and the rewards is known, which we argue is the case with many network applications. In this case, the reward information revealed after an arm-pull for a certain context can also be used to learn about rewards for other contexts as well, since the reward function is already known. The arm-pulls for one context thus help in reducing the exploration requirements of other contexts. Performance of any bandit algorithm is measured in terms of its regret, defined as the difference between the expected cumulative rewards of the Bayes optimal algorithm and that of the agent. Since the bulk of the regret in bandit problems is due to the exploration of different arms, the knowledge of reward function helps in reducing the regret substantially. As previous bandit algorithms cannot exploit this knowledge, we develop new contextual bandit algorithms for this setting. Also, the previously studied contextual bandit problems either assume linear reward mapping \cite{li2010, auer2003linrel, chu2011linucb, agrawal2012thompson} or do not make explicit assumptions about the exact structure of reward functions \cite{langford2008epoch, dudik2011randUCB, agarwal2014monster, lu2010lipschitz, slivkins2014similarity}. It must, however, be noted that the network applications may often involve non-linear reward functions and thus warrant a new approach. The goal of this paper is to study these contextual bandits and develop efficient algorithms for them. 

\subsection{Motivating Examples}
We present some examples of problems arising in communication networks where a contextual bandit formulation is applicable. In each case, there are a few key ingredients: a set of \emph{arms} (typically corresponding to resources such as channels or servers) with associated \emph{attributes} or performance measures (such as throughput capacity, channel gain, processing rate) evolving stochastically with an unknown distribution, and a \emph{context} known to the user, such that the \emph{reward} obtained at each step is a known (possibly non-linear) function of the context and the random attribute of the arm selected in that step. 

\begin{figure}
    \centering
    \includegraphics[width=0.80\textwidth]{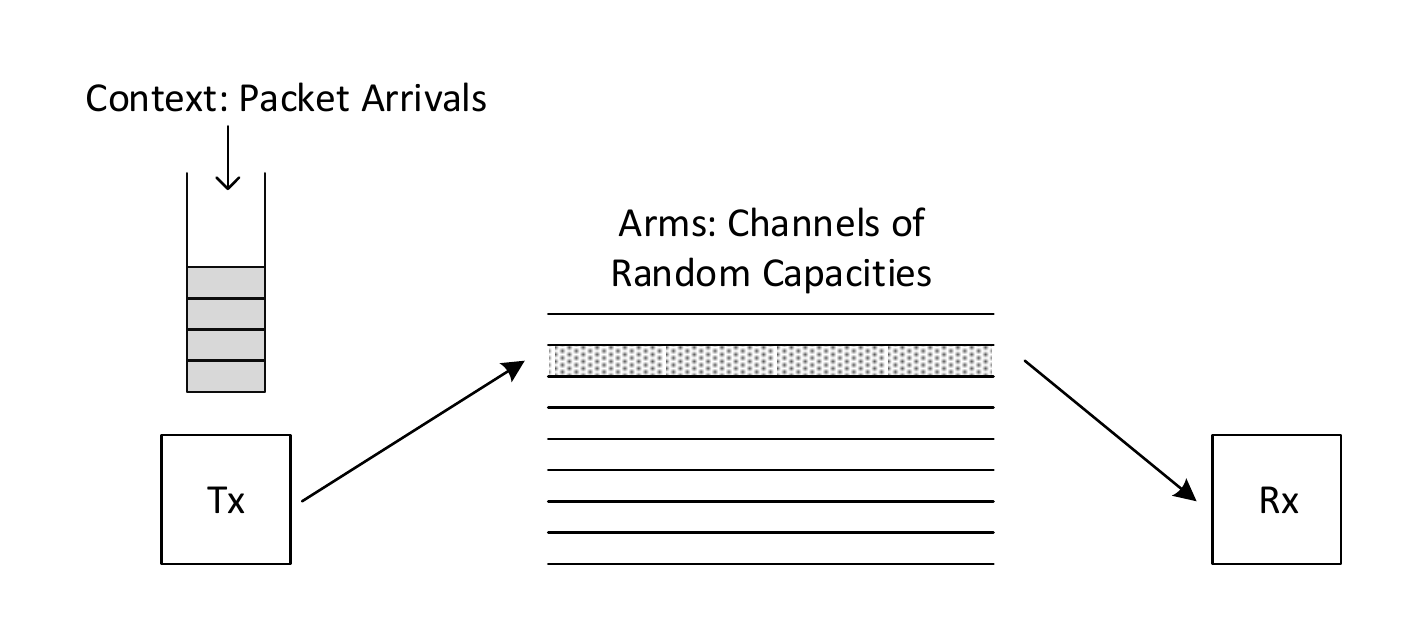}
    \captionsetup{justification=justified,singlelinecheck=false}
    \caption{Dynamic channel selection problem.}
    \label{fig:prob1}
\end{figure}

\begin{figure}
    \centering
    \includegraphics[width=0.80\textwidth]{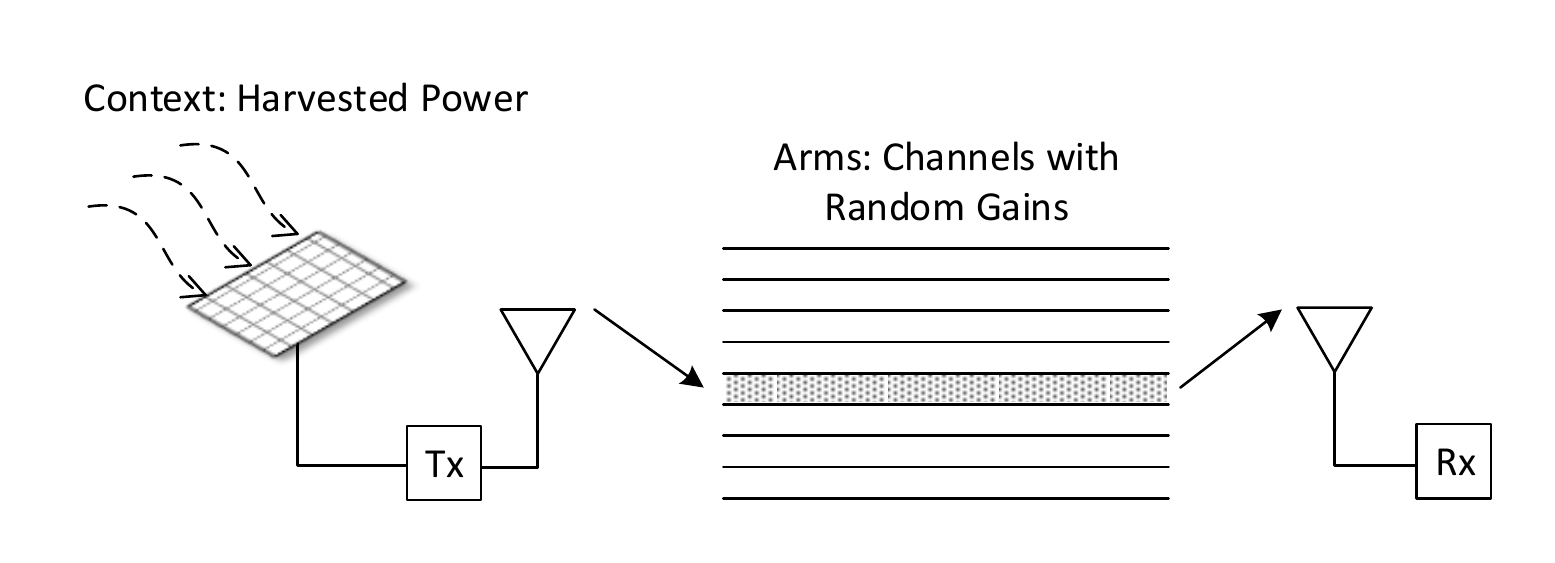}
    \captionsetup{justification=justified,singlelinecheck=false}
    \caption{Power-aware channel selection in energy harvesting communications.}
    \label{fig:prob2}
\end{figure}

\subsubsection{Dynamic Channel Selection} \label{app:ch_sel}
The problem of selecting the best one of several capacity-limited channels every slot based on the available number data packets can be modeled as a contextual bandit problem as shown in figure \ref{fig:prob1}. Here, arms correspond to the different channels available, arm-values or attributes are the instantaneous capacities of these channels which can be assumed to vary stochastically over time (e.g., due to fading or interference in the case of wireless networks) and the context is the number of data-bits that need to be transmitted during the slot. The reward representing the number of successfully transmitted bits over the chosen channel can be expressed as
\begin{equation}
g(y, x) = \min \left\{ y, x \right\},
\end{equation}
where $y$ is the number of data-bits to be transmitted, $x$ is the capacity of the channel. In this problem, channels are to be selected sequentially over the trials based on the context information. Assuming no prior information about the channels, the goal in this problem is to maximize the number of successfully transmitted bits over time. When the number of bits to be sent at each time are finite, this represents the case of discrete contextual bandits considered in this paper.

\begin{figure}
    \centering
    \includegraphics[width=0.80\textwidth]{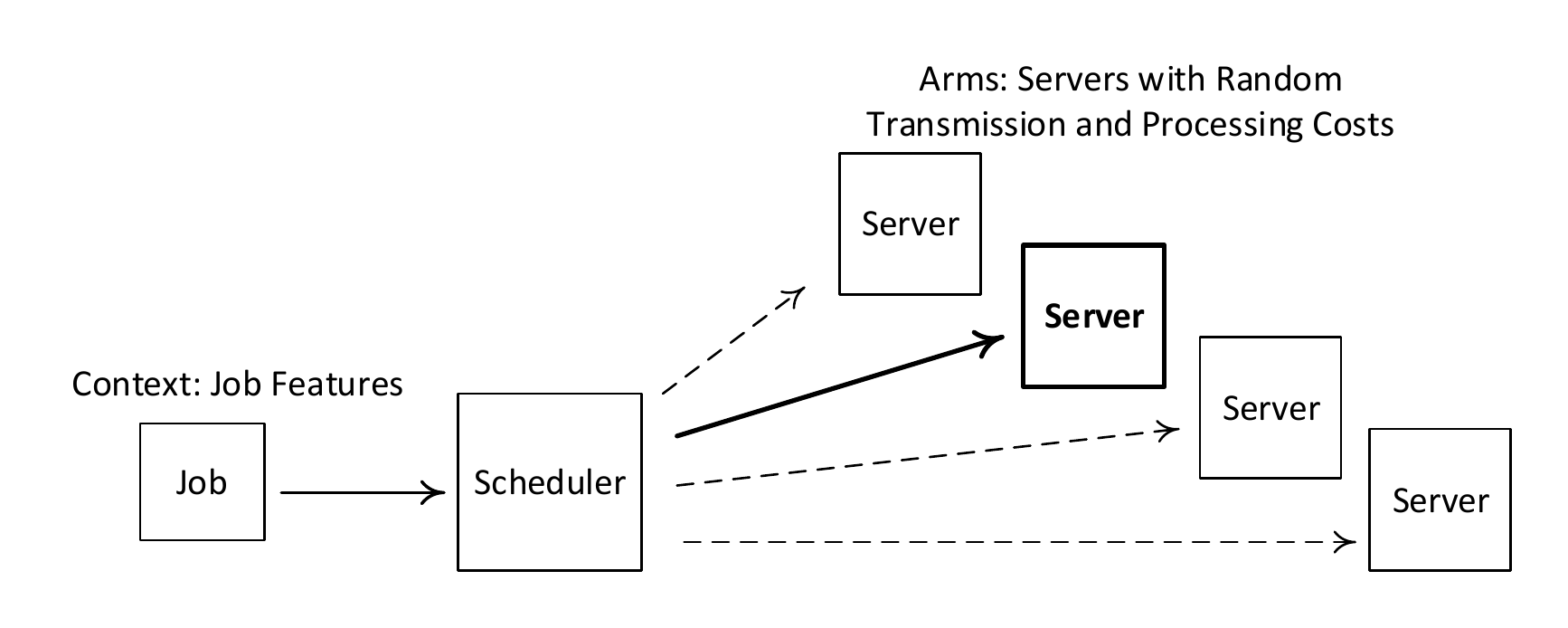}
    \captionsetup{justification=justified,singlelinecheck=false}
    \caption{Job-aware computational offloading problem.}
    \label{fig:prob3}
\end{figure}

\subsubsection{Energy Harvesting Communications} \label{app:harvest}
Consider a power-aware channel selection problem in energy harvesting communications shown in figure \ref{fig:prob2}. In every slot, the current harvested power available for transmission $p$ is known to the transmitter as context. Based on this context information, a channel is to be selected sequentially in each slot for transmission. The reward function in this case is the power-rate function, which can be assumed to be proportional to the AWGN channel capacity $\log(1+ p x)$, where $x$ denotes the instantaneous channel gain-to-noise ratio (the arm attribute) varying stochastically over time. The goal of this problem is to maximize the sum-rate over time. Note that the harvested power $p$ can even be a continuous random variable, and $\log(1+ px)$ is Lipschitz continuous in $p$ for all $x \geq 0$. This represents the case of continuous contextual bandits considered in the paper.

\subsubsection{Multi-Objective Optimization in Computational Offloading}
Computational offloading involves transferring job-specific data to a remote server and further computations on the server. These problems involve rewards that could be functions of multiple objectives such as transmission latency and processing time required on the server. For each server, the reward would depend on various attributes such as the data-rate of the transmission link to the server and the current effective processing rate of the server, which could be unknown stochastic processes, as well as specific features of each job such as the data transmission and processing requirements which act as the context for this problem. As shown in figure \ref{fig:prob3}, this problem of sequential server allocation for various jobs to maximize the cumulative reward over time can thus be formulated as a suitable contextual bandit problem. 

\subsection{Contributions}
We first analyze the case of discrete and finite context-spaces and propose a policy called discrete contextual bandits or DCB($\e$) that requires $O(MK)$ storage, where $M$ is the number of distinct contexts and  $K$ is the number of arms. DCB($\e$) uses upper confidence bounds similar to UCB1, a standard stochastic MAB algorithm from \cite{auer2002} and yields a regret that grows logarithmically in time and linearly in the number of arms not optimal for any context. A key step in our analysis to use a novel proof technique to prove a constant upper bound on the regret contributions of the arms which are optimal for some context. In doing so, we also prove a high probability bound for UCB1 on the number of pulls of the optimal arm in the standard MAB problem. This high probability bound can be independently used in other bandit settings as well. Further, we use the standard MAB asymptotic lower bound results from \cite{lai1985} to show the order optimality of DCB($\e$). Note that DCB($\e$) outperforms UCB1 and Multi-UCB, a discretized version of the Query-ad-clustering algorithm from \cite{lu2010lipschitz}. Regret of UCB1 scales linearly with time as it ignores the context information completely. Multi-UCB uses the context information to run separate instances of UCB1 for each context, but is unable to exploit the reward function knowledge. Regret of Multi-UCB, therefore, grows logarithmically in time, but unlike DCB($\e$) it scales linearly in the total number of arms. 

For continuous context-spaces with Lipschitz rewards, we propose a novel algorithm called continuous contextual bandits or CCB($\e,\d$) that quantizes the context-space and uses DCB($\e$) over the quantized space as a subroutine. Our analysis of CCB($\e, \d$) reveals an interesting trade-off between the storage requirement and the regret, where desired performance can be achieved by tuning the parameter $\d$ to the time horizon. Decreasing the quantization step-size increases the storage requirements while decreasing the regret. Thus by exploiting the reward-function knowledge for joint learning across contexts, CCB($\e, \d$) is able to obtain regrets even smaller than the lower bound of $\Omega(T^{2/3})$ from \cite{lu2010lipschitz} for continuous contextual bandits without such knowledge. For the case of unknown time horizon, we employ a doubling technique that provides similar regret guarantees while using the same amount of storage. CCB($\e, \d$), therefore, empowers the system designer by giving control over $\d$ to tune the regret performance.

\subsection{Organization}
This paper is organized as follows. We first formulate the general contextual bandit problem in section \ref{sec:prob}. In section \ref{sec:discrete}, we discuss various algorithmic approaches to discrete contextual bandits and propose our DCB($\e$) algorithm. Regret analysis of this algorithm is performed in section \ref{sec:analysis} followed by a discussion on the regret lower bound for the discrete case. We extend the algorithm to continuous contexts and analyze the storage vs regret trade-off in section \ref{sec:conti}. Results of numerical simulations are presented in section \ref{sec:simu}. Finally, we summarize our contributions and discuss some avenues for future work in section \ref{sec:conclusion}. The detailed proofs of the theorems are provided in appendices at the end of this paper.

\section{Related Work}
Lai and Robbins \cite{lai1985} wrote one of the earliest papers on the stochastic MAB problem and provided an asymptotic lower bound of $\Omega(K \log T)$ on the expected regret for any bandit algorithm. In \cite{agrawal1995}, sample-mean based upper confidence bound (UCB) policies are presented that achieve the logarithmic regret asymptotically. In \cite{auer2002}, several variants of the UCB based policies, including UCB1, are presented and are proved to achieve the logarithmic regret bound uniformly over time for arm-distributions with finite support. 

Contextual bandits extend the standard MAB setting by providing some additional information to the agent about each trial. In the contextual bandits studied in \cite{li2010, auer2003linrel, chu2011linucb, agrawal2012thompson}, the agent observes feature vectors corresponding to each arm. The expected reward of the chosen arm is assumed to be a linear function of its corresponding context vector. LinRel \cite{auer2003linrel} and LinUCB \cite{chu2011linucb} use the context features to estimate the mean reward and its corresponding confidence interval for each arm. The arm that achieves the highest UCB which is the sum of mean reward and its confidence interval is pulled. These confidence intervals are defined in such a way that the expected reward lies in the confidence interval around its estimated mean with high probability. In LinRel and LinUCB, mean and covariance estimates of the linear function parameters are used to estimate the upper confidence bounds for the arm rewards. LinUCB algorithm has also been empirically evaluated on Yahoo!'s news recommendation database in \cite{li2010}. An algorithm based on Thompson sampling is proposed in \cite{agrawal2012thompson} for bandits with linear rewards. In this Bayesian approach, the agent plays an arm based on its posterior probability of having the best reward. The posterior probabilities are then updated according to the reward obtained after each pull. LinRel and LinUCB achieve a regret bound of $\tilde{O}(\sqrt{T})$ while the Thompson sampling algorithm achieves a bound of $\tilde{O}(\sqrt{T^{1+\e}})$ for any $\e > 0$, where $T$ is the time horizon. Since these algorithms consider linear functions and store estimates of the parameters characterizing the rewards, their storage requirements do not increase with time.

More general contextual bandits that do not assume any specific relation between the context and reward vectors are studied in \cite{langford2008epoch, dudik2011randUCB, agarwal2014monster}. The epoch-greedy algorithm \cite{langford2008epoch} separates the exploration and exploitation steps and partitions the trials in different epochs. Each epoch starts with one exploration step followed by several exploitation steps. It stores the history of contexts, actions and rewards for the exploration steps and requires $O(T^{2/3})$ storage. Its regret is also upper bounded by $\tilde{O}(T^{2/3})$. RandomizedUCB \cite{dudik2011randUCB}, however, does not make any distinction between exploration and exploitation. At each trial it solves an optimization problem to obtain a distribution over all context-action mappings and samples an arm as the action for the trial. Its regret bound is $\tilde{O}(\sqrt{T})$, but it needs $O(T)$ storage as it stores the history of all trials. ILTCB \cite{agarwal2014monster} is a more efficient algorithm that solves an optimization problem only on a pre-specified set of trials. However, it still needs $O(T)$ storage. Since these general contextual bandit algorithms do not make any assumption on the context-reward mapping, they store the raw history of trials and do not summarize the information into meaningful statistics. Storage requirements for these approaches, therefore, increase with time.

Contextual bandits in general metric spaces are studied in \cite{lu2010lipschitz, slivkins2014similarity} under the assumption of Lipschitz reward function. This formulation is also able to handle the case of continuous contexts unlike the previously discussed general contextual bandit problems. Query-ad-clustering algorithm \cite{lu2010lipschitz} partitions the metric spaces in uniform clusters, whereas the meta-algorithm from \cite{slivkins2014similarity} considers adaptive partitions and uses a general bandit algorithm as a subroutine. Upper confidence indexes are evaluated knowing the corresponding clusters of incoming contexts, the clusters-wise empirical means of previously obtained rewards and the number of pulls for each cluster. Regret bounds for these approaches involve the packing and covering dimensions of the context and arm spaces.

Recently, contextual bandits with budget and time constraints have been studied in \cite{badan2014, wu2015ccb}. Resourceful contextual bandits from \cite{badan2014} consider a general setting with random costs associated with the arms, a continuous context-space and multiple budget constraints. An algorithm called Mixture\_Elimination that achieves $O(\sqrt{T})$ regret is proposed for this problem. A simplified model with fixed costs, discrete and finite contexts, exactly one time constraint and one budget constraint is considered in \cite{wu2015ccb}. For this model, an algorithm called UCB\_ALP achieving $O(\log T)$ regret is proposed. In these problem formulations, time and budget constraints also affect the regular exploration-exploitation trade-off.

\section{Problem Formulation} \label{sec:prob}
In the general stochastic contextual bandit problem, there is a distribution $D$ over $(\bfy, \bfr)$, where $\bfy \in \sY$ is a context vector in context space $\sY$ and $\bfr \in \bbR^{K}$ is the reward vector containing entries corresponding to each arm. The problem is a repeated game such that at each trial $t$, an independent sample $(\bfy_{t}, \bfr_{t})$ is drawn from $D$, context vector $\bfy_{t}$ is presented to the agent. Once the agent chooses an arm $a_{t} \in \{1, \cdots, K\}$, it receives a reward $r_{a,t}$ which is the $a_{t}$-th element of the vector. A contextual bandit algorithm $\sA$ chooses an arm $a_{t}$ to pull at each trial $t$, based on the previous history of contexts, actions, their corresponding rewards: $(\bfy_{1}, a_{1}, r_{a,1}),\cdots, (\bfy_{t-1}, a_{t-1}, r_{a,t-1})$ and the current context $\bfy_{t}$. The goal is to maximize the total expected reward $\sum_{t=1}^{T} \bbE_{(\bfy_t, \bfr_t) \sim D} \left[ r_{a,t} \right]$. 

Consider a set $\sH$ consisting of hypotheses $h:\sY \rightarrow \{1, \cdots, K\}$. Each hypothesis maps the context $\bfy$ to an arm $a$. The expected reward of a hypothesis $h$ is expressed as
\begin{equation}
    R(h) = \bbE_{(\bfy, \bfr) \sim D} \left[ r_{h(\bfy)} \right].
\end{equation}
The Bayes optimal hypothesis $h^{*}$ refers to the hypothesis with maximum expected reward $h^{*} = \underset{h\in \sH}{\argmax}\ R(h)$ and maps contexts as $h^{*}(\bfy) = \underset{a \in \{1,\cdots, K\}}{\argmax} \bbE \left[ r_{a} \mid \bfy \right]$. Agent's goal is to choose arms sequentially and compete with the hypothesis $h^{*}$. Regret of the agent's algorithm $\sA$ is defined as the difference in expected reward accumulated by the best hypothesis $h^{*}$ and that by its algorithm over different trials, which can be expressed as
\begin{equation}
    \mathfrak{R}_{\sA} (n) = n R(h^{*}) - \bbE\left[ \sum_{t=1}^{n} r_{\sA(t),t} \right], \label{eq:def_regret}
\end{equation}
where $\sA(t)$ denote the arm pulled by algorithm $\sA$ in $t$-th trial. Note that the expectation in (\ref{eq:def_regret}) is over the sequence of random samples from $D$ and any internal randomization potentially used by $\sA$. The reward maximization problem is equivalently formulated as a regret minimization problem. Smaller the regret $\mathfrak{R}_{\sA} (n)$, better is the policy. We would like the regret to be sublinear with respect to time $n$ so that the time-averaged regret will converge to zero. These policies are asymptotically Bayes optimal, since $\underset{n\rightarrow 0}{\lim} \frac{\mathfrak{R}_{\sA} (n)}{n} = 0$.

Let $x_{a,t}$ denote the state or the value of $a$-th arm at $t$-th slot and $\sX$ denote the set of all possible states of an arm. We assume that the reward is a known bounded function of the selected arm's state and the context represented as
\begin{equation}
r_{a,t} = g(\bfy_{t}, x_{a,t}), \label{eq:def_reward}
\end{equation}
where $ \abs{g (\bfy,x)} \leq B$ for some $B > 0$ and $\forall \bfy \in \sY, x \in \sX$. Note that, in the standard MAB problems, reward of an arm is same as its state. We assume semi-bandit feedback, which means only the state of the agent's chosen arm is revealed. If the reward is an injective function of the arm state for all contexts, then the arm state can be inferred without ambiguity from the reward observation, since the current context is already known to the user. Thus, from now on we assume that the chosen arm's state is either revealed by semi-bandit feedback or inferred from the reward observation. In this problem formulation, the context space $\sY$ can be either discrete or continuous. We first develop an algorithm called DCB($\e$) for discrete context spaces in section \ref{sec:discrete} and later propose its extension CCB($\e, \d$) for continuous context spaces in section \ref{sec:conti} assuming that the reward function is Lipschitz with respect to the context variable.

\section{Discrete Contexts} \label{sec:discrete}
In this section, we consider discrete and finite context space $\sY = \{ \bfy_{(1)},\cdots, \bfy_{(M)} \}$ with $M$ elements. As a convention, we use $j$ and $k$ to index the arms and $i$ to index the discrete contexts.

\subsection{Naive Approach}
In the standard MAB setting, the reward of an arm is same as its value and all trials are stochastically identical. UCB1 policy, proposed by Auer \textit{et al.} in \cite{auer2002}, pulls the arm that maximizes $ \overline{X}_k + \sqrt{2 \ln n / m_k}$ at $n$-th slot, where $\overline{X}_k$ is the mean observed value of the $k$-th arm and $m_k$ is the number of pulls of $k$-th arm. The second term is the size of the one-sided confidence interval for the empirical mean within which the actual expected reward value falls with high probability. Applying such standard MAB algorithms directly to our problem will result in poor performance, as they ignore the context-reward relationship. Playing the arm with highest mean value is not necessarily a wise policy, since the trials are not all identical. When every context in $\sY$ appears with a non-zero probability, playing the same arm over different trials implies that a suboptimal arm is pulled with a non-zero probability at each trial. The expected regret, therefore, grows linearly over time. Although the naive approach is not directly applicable, it serves as one of the benchmarks to compare our results against.

\subsection{Multi-UCB}
Main drawback of the plain UCB1 policy is that it ignores the contexts completely. One way to get around this problem is to run a separate instance of UCB1 for each context. Each UCB1 can be tailored to a particular context, where the goal is to learn the arm-rewards for that context by trying different arms over time. UCB1 instance for context $\bfy_{i}$ picks the arm that maximizes $\overline{g_{i,k}} + \sqrt{2 \ln n_{i} / m_{i,k}}$,  where $\overline{g_{i,k}}$ denotes the empirical mean of the rewards obtained from previous $m_{i,k}$ pulls of $k$-th arm for context $\bfy_{i}$ and $n_{i}$ denotes the number of occurrences of that context till the $n$-th trial. Since each instance incurs a regret that is logarithmic in time, the overall regret of Multi-UCB1 is also logarithmic in time. Compared to the naive approach where the regret is linear in time, this is a big improvement. Notice that Multi-UCB learns the rewards independently for each context and does not exploit the information revealed by the pulls for other contexts. Unless a single arm is optimal for all contexts, every arm's contribution to the regret is logarithmic in time, since they are explored independently for all contexts. Hence the regret scales as $O(M K \log n)$.

\subsection{New Policy}
Let $\th_{i,j}$ be the expected reward obtained by pulling the $j$-th arm for context $\bfy_{(i)}$ which is evaluated as
\begin{equation}
\th_{i,j} = \bbE\left[ r_{j} \mid \bfy_{(i)} \right] = \bbE \left[g(\bfy_{(i)}, x_{j}) \right].
\end{equation}
Let the corresponding optimal expected reward be $\th_{i}^{*} = \th_{i,h^{*}(\bfy_{(i)})}$. We define $G_{i}$, the maximum deviation in the rewards, as 
\begin{equation}
    G_{i} = \underset{x \in \sX}{\sup}\ g(\bfy_{(i)}, x ) - \underset{x \in \sX}{\inf}\ g(\bfy_{(i)}, x ).
\end{equation}
Since the reward function is assumed to be bounded over the domain of arm values, $G_{i} \leq 2 B$.

Similar to Multi-UCB, the idea behind our policy is to store an estimator of $\th_{i,j}$ for every context-arm pair. For these statistics, we compute upper confidence bounds similar to UCB1 \cite{auer2002} depending on the number of previous pulls of each arm and use them to decide current action. Since the selected arm can be observed or inferred irrespective of the context, it provides information about possible rewards for other contexts as well. This allows joint learning of arm-rewards for all contexts. Intuitively, this policy should be at least as good as Multi-UCB. Our policy, shown in algorithm \ref{algo:DCB_espilon}, is called as Discrete Contextual Bandit or DCB($\e$). In table \ref{tab:notations}, we summarize the notations used.

\begin{algorithm}
\caption{DCB($\e$)}
\label{algo:DCB_espilon}
\begin{algorithmic}[1]
\STATE  {\bf Parameters:} $\e > 0$.
\STATE  {\bf Initialization:} $(\hat{\th}_{i,j})_{M\times K} = (0)_{M\times K}, (m_{j})_{1\times K} = (0)_{1\times K}$.
\FOR {$n = 1$ to $K$} 
\STATE Set $j = n$ and pull $j$-th arm;
\STATE Update $\hat{\th}_{i,j}$, $\forall i: 1 \leq i \leq M$, as $\hat{\th}_{i,j} = g(\bfy_{(i)}, x_{j,n})$ and $m_{j}$ as $m_{j} = 1$;
\ENDFOR
\STATE // MAIN LOOP
\WHILE {1}
\STATE $n = n +1$;
\STATE Given the context $\bfy_{n} = \bfy_{(i)}$, pull the arm $\sA(n)$ that solves the maximization problem 
\begin{equation}
\sA(n) = \argmax\limits_{j \in \sS} \left\{ \hat{\th}_{i,j} + G_{i} \sqrt{\frac{(2+\e) \ln n}{m_{j}}} \right\};
\end{equation}
\STATE Update $\hat{\th}_{i,j}$ and $m_{j}$, $\forall i: 1 \leq i \leq M$ as
\begin{align}
\hat{\th}_{i,j}(n) &= \left\{ 
\begin{array}{l l}
\frac{\hat{\th}_{i,j}(n-1) m_{j}(n-1) + g(\bfy_{(i)},x_{j,n})}{m_{j}(n-1)+1}, & \text{if } j = \sA(n)  \\
\hat{\th}_{i,j}(n-1), & \text{else}
\end{array}
\right. \label{eq:update_theta}\\
m_{j}(n) &= \left\{ 
\begin{array}{l l}
m_{j}(n-1)+1, & \text{if } j = \sA(n)  \\
m_{j}(n-1); & \text{else}
\end{array}
\right. \label{eq:update_m} 
\end{align} 
\ENDWHILE
\end{algorithmic}
\end{algorithm} 

\begin{table}[htbp]\caption{Notations for Algorithm Analysis}
\centering
\begin{tabular}{l c p{10cm} }
\toprule
$K$ &:& number of arms \\
$M$ &:& number of distinct contexts (for discrete contexts)\\
$\sY$ &:& set of all contexts \\
$\sX$ &:& support set of the arm-values \\
$*$ &:& index to indicate the contextually optimal arm \\
$p_{i}$ &:& probability of the context being $\bfy_{(i)}$ (for discrete contexts) \\
$\sY_{j}$ &:& set of all contexts for which $j$-th arm is optimal \\
$q_{j}$ &:& sum probability of all contexts in $\sY_{j}$ \\

$\sS$ &:& index set of all arms \\
$\sO$ &:& index set of all optimal arms \\
$\overline{\sO}$ &:& index set of all non-optimal arms \\
$T_{j}^{O}(n)$ &:& number of \textit{optimal} pulls of $j$-th arm in $n$ trials \\
$T_{j}^{N}(n)$ &:& number of \textit{non-optimal} pulls of $j$-th arm in $n$ trials \\
$T_{j}(n)$ &:& number of pulls of $j$-th arm in $n$ trials (Note:  $T_{j}(n) = T_{j}^{O}(n) + T_{j}^{N}(n)$)\\
$\sA(n)$ &:& index of the arm pulled in $n$-th trial by DCB($\e$) \\

\bottomrule
\end{tabular}
\label{tab:notations}
\end{table}

Algorithmically, DCB($\e$) is similar in spirit to UCB1. It differs from UCB1 in updating multiple reward averages after a single pull. Also, the reward-ranges $G_{i}$ for various contexts are in general different and better regret constants are obtained by scaling the confidence bounds accordingly. The major difference, however, comes from the parameter $\e$ which needs to be strictly positive. This condition is essential for bounding the number of pulls of the arms that are optimal for some context. The analysis for these arms is sophisticated and differs significantly from the standard analysis of UCB1, and is therefore one of the main contributions of our paper.

DCB($\e$) uses an $M \times K$ matrix $(\hat{\th}_{i,j})_{M\times K}$ to store the reward information obtained from previous pulls. $\hat{\th}_{i,j}$ is the sample mean of all rewards corresponding to context $\bfy_{(i)}$ and the observed $x_{j}$ values. In addition, it uses a length $K$ vector $(m_{j})_{1 \times K}$ to store the number of pulls of $j$-th arm up to the current trial. At the $n$-th trial, $\sA(n)$-th arm is pulled and $x_{\sA(n),n}$ is revealed or inferred from the reward. Based on this, $(\hat{\th}_{i,j})_{M\times K}$ and $(m_{i})_{1 \times K}$ are updated. It should be noted that the time indexes in (\ref{eq:update_theta}) and (\ref{eq:update_m}) are only for notational clarity. It is not necessary to store the previous matrices while running the algorithm. Storage required by the DCB($\e$) policy is, therefore, only $\Theta(MK)$ and does not grow with time. In section \ref{sec:analysis}, we analyze and upper bound the cumulative regret of our policy and show that it scales logarithmically in time and linearly in the number of arms not optimal for any context.

\section{Regret Analysis of DCB($\e$)} \label{sec:analysis}
In the standard MAB problems, regret arises when the user pulls the non-optimal arms. Hence the regret upper bounds can be derived by analyzing the expected number of pulls of each non-optimal arm. In our problem, however, there can be multiple contextually optimal arms. It is, therefore, important to analyze not only the number of times an arm is pulled but also the contexts it is pulled for.

Let $\sS = \{1,2,..., K\} $ denote the index set of all arms. Let us define an optimal set $\sO$ and a non-optimal set $\overline{\sO}$ as
\begin{align}
\sO &= \{ j |\ \exists \bfy_{(i)} \in \sY: h^{*} (\bfy_{(i)}) = j \}, \nn \\
\overline{\sO} & = \sS \backslash \sO.
\end{align}
Note that the arms in the non-optimal set $\overline{\sO}$ are not optimal for any context. This means that every time an arm in $\overline{\sO}$ is pulled, it contributes to the expected regret. However, an arm in $\sO$ contributes to the expected regret only when it is pulled for contexts with different optimal arms. When these optimal arms are pulled for the contexts for which they are optimal, they don't contribute to the expected regret. We analyze these two cases separately and provide upper bounds on the number of pulls of each arm for our policy. 

\begin{theorem}[Bound for non-optimal arms] \label{thm:non_opt_T_UB}
For all $j \in \overline{\sO}$, under the DCB($\e$) policy 
\begin{equation}
\bbE \left[ T_{j}(n) \right] \leq \frac{4 (2+\e)\ln n}{\min\limits_{1 \leq i \leq M} \left(\D_{j}^{(i)} \right)^2} + 1 + \frac{\pi^2}{3},
\end{equation}
where $\D_{j}^{(i)} = \th_{i}^{*} - \th_{i,j}$.
\end{theorem}

\begin{theorem}[Bound for optimal arms] \label{thm:opt_T_UB}
For all $j \in \sO$, under the DCB($\e$) policy 
\begin{equation}
\bbE \left[ T_{j}^{N}(n) \right] \leq n_{o} + \left( \sum\limits_{\bfy_{(i)} \in \sY_{j}} \frac{2}{p_{i}^{2}} \right) + 2^{3+2\e}K^{4+2\e} c_{\e} \left( \sum_{\bfy_{(i)} \in \sY_{i}} \frac{1}{\left(p_{i}\right)^{2+2\e}} \right) +  \frac{\pi^{2}}{3} \left( \sum\limits_{\bfy_{(i)} \notin \sY_{j}} p_{i} \right),
\end{equation}
where $c_{\e} = \sum_{n = 1}^{\infty} \frac{1}{n^{1+2\e}}$ and $n_{o}$ is the minimum value of $n$ satisfying
\begin{equation}
    \Big\lfloor \frac{p_{o} n}{2 K} \Big\rfloor > \Big\lceil \frac{4 (2+\e) \ln n}{\left(\Delta_{o}\right)^{2}} \Big\rceil,
\end{equation}
with $p_{o} = \min\limits_{1 \leq i \leq M} p_i$, and $\D_{o} = \min\limits_{\forall i,j : h^{*}(i) \neq j} \D_{j}^{(i)}$.
\end{theorem}

Theorem \ref{thm:non_opt_T_UB}, whose proof can be found in appendix \ref{apx:non_opt_arm}, provides an upper bound on the number of pulls of any non-optimal arm that scales logarithmically in time. As discussed previously, all such pulls contribute to the regret. Theorem \ref{thm:opt_T_UB} states that the number of non-optimal pulls of any optimal arm is bounded above by a constant. Regret contribution of the optimal arms is, therefore, bounded. We sketch the proof of theorem \ref{thm:opt_T_UB} to emphasize the basic ideas of our novel proof technique and provide the detailed proof in appendix \ref{apx:opt_arm_bound}.

In order to bound the number of pulls of the optimal arms in DCB($\e$), we require a high probability bound on the optimal pulls by UCB1 \cite{auer2002} in the standard MAB problem. Since we are not aware of any such existing result, we derive one ourselves. Lemma \ref{lem:high_prob_ucb} shows a bound for the generalized UCB1 policy called UCB1($\e$), proof of which can be found in appendix \ref{apx:high_prob_ucb}. Note that the confidence interval in UCB1($\e$) is $\sqrt{\frac{(2+\e)\ln n}{m_{k}}}$, which reduces to that of UCB1 for $\e = 0$. In this context, $\m_{j}$ denotes the mean value of $j$-th machine and $\m^{*} = \max\limits_{1 \leq j \leq K} \m_{j}$.

\begin{lemma} [High probability bound for UCB1($\e$)] \label{lem:high_prob_ucb}
In the standard stochastic MAB problem, if the UCB1($\e$) policy is run on $K$ machines having arbitrary reward distributions with support in $[0,1]$, then
\begin{equation}
\pr \left\{ T^{*}(n) < \frac{n}{K}\right\} < \frac{2K^{4+2\e}}{n^{2+2\e}}, \label{eq:lem_high_prob}
\end{equation}
for all $n$ such that
\begin{equation}
\left\lfloor \frac{n}{K}\right\rfloor > \frac{4(2+\e) \ln n}{\min\limits_{j \neq j^{*}} (\m^{*} - \m_j)^{2}} \label{eq:cond_ucb}.
\end{equation}
\end{lemma}

Note that ${\frac{n}{\ln n}}$ is an increasing sequence in $n$ and there exists some $n' \in \bbN$ such that condition (\ref{eq:cond_ucb}) is satisfied for every $n \geq n'$. $n'$ is a constant whose actual value depends on the true mean values for the arms.

\begin{proof}[Proof sketch of theorem \ref{thm:opt_T_UB}]
For every optimal arm, there exists a non-zero probability of appearance of the contexts for which it is optimal. Over successive pulls, DCB($\e$) leads us to pull the optimal arms with very high probability for their corresponding contexts. Thus, these arms are pulled at least a constant fraction of time with high probability. Since the constant fraction of time is much more than the logarithmic exploration requirement, these optimal arms need not be explored during other contexts. The idea is to show that the expected number of these non-optimal pulls is bounded.

Let us define $K$ mutually exclusive and collectively exhaustive context-sets $\sY_{j}$ for $j \in \sS$ as
\begin{equation}
\sY_{j} = \{ \bfy_{(i)} \in \sY \mid h^{*}(i) = j \}. \nn
\end{equation}
Further, let $N_{i}(n)$ denote the number of occurrences of context $\bfy_{(i)}$ till the $n$-th trial. At the $n$-th trial, if an optimal arm $j$ is pulled for a context $\bfy_{n} = \bfy_{(i)} \notin \sY_{j}$, this happens due to one of the following reasons:
\begin{enumerate}
\item The number of pulls of $j$-th arm till the $n$-th trial is small and we err due to lack of sufficient observations of $j$-th arm's values, which means the arm is {\it under-played}.
\item We err in spite of having previously pulled $j$-th arm enough number of times.
\end{enumerate}
We bound the probability of occurrence for these two cases by functions of $n$, whose infinite series over $n$ is convergent. These reasons can be analyzed in three mutually exclusive cases as follows

\paragraph{Under-realized contexts} In this case, the number of occurrences of contexts in $\sY_{j}$ till the $n$-th trial is small. If all the contexts in $\sY_{j}$ are under-realized, this could lead to $j$-th arm being under-explored at trial $n$. We use Hoeffding's inequality (see appendix \ref{apx:inequalities}) to bound the expected number of occurrences of this event.

\paragraph{Under-exploited arm} We assume that no context $\bfy_{(i)} \in \sY_{j}$ is under-realized and yet $j$-th arm is not pulled in enough trials. For these contexts $j$-th arm is optimal, but we still don't pull it often enough. In order to bound the expected number of occurrences of this event, we use the high probability bound from lemma \ref{lem:high_prob_ucb}.

\paragraph{Dominant confidence bounds} This case corresponds to the event where no context in $\sY_{j}$ is under-realized and $j$-th arm is also pulled sufficiently often, yet it is pulled non-optimally. This occurs when the upper confidence bound for some other arm dominates the corresponding DCB($\e$) index.

We prove the theorem by upper bounding the expected number of occurrences of these events over all trials.
\end{proof}

\subsection{Upper Bound on Regret}
Upper bounds on the number of pulls of all arms in theorems \ref{thm:non_opt_T_UB} and \ref{thm:opt_T_UB} lead us to the regret bound in theorem \ref{thm:reg_UB}. It states a regret upper bound that grows linearly in the number of non-optimal arms and logarithmically in time, i.e. $O(\left| \overline{\sO} \right| \ln n)$.

\begin{theorem}[Regret bound for DCB($\e$)] \label{thm:reg_UB}
The expected regret under the DCB($\e$) policy till trial $n$ is at most
\begin{equation}
\frac{4 (2+\e)\D_{\max}}{\left(\D_{o} \right)^2} \left| \overline{\sO} \right| \ln n + O(1),
\end{equation}
where $\D_{\max} = \underset{\forall i,j }{\max}\ \D_{j}^{(i)}$.
\end{theorem}
\begin{proof}[Proof of theorem \ref{thm:reg_UB}]
Since the expected number of pulls of optimal arms is bounded (from theorem \ref{thm:opt_T_UB}), their regret contribution is also bounded. Note that this regret bound is still a function of $K$, but a constant in terms of $n$. Using theorem \ref{thm:non_opt_T_UB} and the definition of $\D_{o}$, total number of pulls of any non-optimal arm till $n$-th trial is upper bounded by $\frac{4 (2+\e)\D_{\max} \ln n}{\left(\D_{o} \right)^2} + O(1)$. As the contribution to expected regret of any non-optimal pull is upper bounded by $\D_{\max}$, considering all non-optimal arms proves the theorem.
\end{proof}

\begin{remark}\label{rem:eps_req}
Note that in order for theorem \ref{thm:opt_T_UB} to hold, DCB($\e$) must be run with a strictly positive $\e$. For $\e = 0$, the confidence bound looks similar to the one from UCB1 \cite{auer2002}. It remains to be seen if the condition $\e > 0$ is an artifact of our proof technique or a stringent requirement for constant regret bound in theorem \ref{thm:opt_T_UB}. 
\end{remark}

\begin{remark} \label{rem:generalization}
If the reward function is not dependent on the context, this formulation reduces to the standard MAB problem. In this special case, our DCB($\epsilon$) algorithm reduces to UCB1. Thus, our results generalize the results from \cite{auer2002}.
\end{remark}

\begin{remark} \label{rem:dsee}
Strategies like deterministic sequencing of exploration and exploitation (DSEE) \cite{vakili2013}, $\e$-greedy \cite{auer2002} and epoch greedy \cite{langford2008epoch} separate exploration from exploitation. All the arms are explored either deterministically equally in DSEE or uniformly at random in others. Using these strategies for our problem leads to logarithmically scaling regret contributions for all the arms. Since the optimal arms are not known a priory, these strategies do not reduce the regret contribution from optimal arms. 
\end{remark}

\subsection{Asymptotic Lower bound} \label{sec:LB}
For the standard MAB problem, Lai and Robbins provided an asymptotic lower bound on the regret in \cite{lai1985}. For every non-optimal arm $j$ and the families of arm-distributions parametrized by a single real number, they showed:
\begin{equation}
\liminf\limits_{n \rightarrow \infty} \frac{\bbE\{ T_{j}(n)\} }{\ln n} \geq D_{j},
\end{equation}
where $D_{j} > 0$ is a function of the Kullback-Leibler divergence between the reward distribution of $j$-th arm and that of the optimal arm. This result was extended by Burnetas and Katehakis to distributions indexed by multiple parameters in \cite{burn1997}. It is important to note that the lower bound holds only for consistent policies. A policy $\bm{\pi}$ is said to be consistent if $\mathfrak{R}_{\bm{\pi}}(n) = o(n^{\a})$ for all $\a >0$ as $n \rightarrow \infty$. Any consistent policy, therefore, pulls every non-optimal arm at least $\Omega(\ln n)$ times asymptotically. 

The contextual bandit problem can be visualized an interleaved version of several MAB problems, each corresponding to a distinct context. These MAB instances are, however, not independent. When the reward function is known, information gained from an arm pull for any context can be used to learn about the rewards for other contexts as well. In terms of learning, there is no distinction among the pulls of a specific arm for distinct contexts. The burden of exploration, thus, gets shared by all the contexts. Let $N_{i}(n)$ denote the number of occurrences of $\bfy_{(i)}$ till the $n$-th trial. Under the conditions from \cite{lai1985}, for a non-optimal arm $j$ under a consistent policy we get
\begin{equation}
    \liminf\limits_{n \rightarrow \infty} \frac{\bbE\{ T_{j}(n)\} }{\ln N_{i}(n)} \geq D_{i,j} \label{eq:lower_cond1},
\end{equation}
where $D_{i,j} > 0$ depends on the KL divergence between the conditional reward distribution of $j$-th arm and that of the contextually optimal arm. By the law of large numbers, $N_{i}(n) \rightarrow p_{i} n$ as $n \rightarrow \infty$. Thus, we write
\begin{equation}
    \lim\limits_{n \rightarrow \infty} \frac{\ln N_{i}(n)}{\ln n} = 1 + \lim\limits_{n \rightarrow \infty} \frac{\ln \left(N_{i}(n)/n \right) }{\ln n} = 1 \nn.
\end{equation}
Condition (\ref{eq:lower_cond1}), therefore, reduces to 
\begin{equation}
    \liminf\limits_{n \rightarrow \infty} \frac{\bbE\{ T_{j}(n)\} }{\ln n} \geq D_{i,j} \nn .
\end{equation}
Note that there is one such condition corresponding to each context and non-optimal arm pair. Combining these conditions for a non-optimal arm $j$, we get 
\begin{equation}
    \liminf\limits_{n \rightarrow \infty} \frac{\bbE\{ T_{j}(n)\} }{\ln n} \geq \max\limits_{1 \leq i \leq M}  D_{i,j} > 0 \label{eq:lower_cond3}.
\end{equation}
This gives an asymptotic lower bound of $\Omega(\ln n)$ on the number of pulls of a non-optimal arm under a consistent policy. Note that this bound matches the upper bound of $O(\ln n)$ for DCB($\e$) policy proving its order optimality.

\section{Continuous Contexts} \label{sec:conti}
In this section, we extend our algorithm for discrete contexts to the continuous case. For simplicity, we assume that $\sY \subseteq \bbR$. This can be easily extended to general metric spaces similar to \cite{lu2010lipschitz, slivkins2014similarity}. We additionally assume the reward function to be $L$-Lipschitz in context, which means
\begin{equation}
    \left| g(y, x) - g(y', x) \right| \leq L \left| y - y' \right|, \quad \forall y, y' \in \sY, x\in \sX.  \label{eq:lipschitz}
\end{equation}

Our main idea is to partition the support of context space into intervals of size $\d$. We learn and store the statistics from the discrete-context algorithms for the center points of each interval. We quantize each incoming context value $y_{t}$ is into $\hat{y}_{t}$ and use DCB($\e$) as a subroutine to perform successive pulls. Since the context-distribution is not known, we resort to uniform quantization. Algorithm \ref{algo:CCB} presents our continuous contextual bandit algorithm based on these ideas. Theorem \ref{thm:CCB_bound} states its regret upper bound in terms of the time horizon $T$.

\begin{algorithm}
\caption{CCB($\e, \d$)}
\label{algo:CCB}
\begin{algorithmic}[1]
\STATE  {\bf Parameters:} $\e, \d > 0$.
\STATE Partition $\sY$ into intervals of size $\d$.
\STATE Set up a DCB($\e$) instance with the set $\hat{\sY}$ containing the center points of each partition as the context-set.
\FOR {$t = 1$ to $T$} 
\STATE Quantize $y_{t}$ into $\hat{y}_{t}$;
\STATE Feed $\hat{y}_{t}$ as the context to DCB($\e$) and pull the arm it recommends;
\STATE Update the internal statistics of DCB($\e$);
\ENDFOR
\end{algorithmic}
\end{algorithm}

\begin{theorem}[Regret bound for CCB($\e,\d$)] \label{thm:CCB_bound}
The expected regret under CCB($\e,\d$) policy is upper bounded by $\d LT + O(\log T)$, where $T$ is the time horizon.
\end{theorem}

\begin{proof}
Let the instance of DCB($\e$) algorithm working with the quantized contexts $\hat{y}_{t}$ be $\sA$ and its expected regret be $\mathfrak{R}_{\sA}(T)$. Similarly, let $\sB$ denote the instance of CCB($\e,\d$) algorithm and $\mathfrak{R}_{\sB}(T)$ its expected regret. The Bayes optimal policies on $\sY$ and $\hat{\sY}$ are denoted by $h^{*}$ and $\hat{h}^{*}$, respectively. It is important to note that $\hat{h}^{*}(\hat{y}) = h^{*}(\hat{y})$ for all $\hat{y} \in \hat{\sY}$ and $\sA(t) = \sB(t)$ since $\sB$ follows the recommendations of $\sA$. The main difference is the input context for $\sA$ and $\sB$, which changes the optimal rewards. The regret of $\sA$, in terms of these notations, is 
\begin{align}
    \mathfrak{R}_{\sA}(T) &= \sum_{t=1}^{T} \bbE \left[ g(\hat{y}_{t}, x_{\hat{h}^{*}(t),t})  - g(\hat{y}_{t}, x_{\sA(t),t})\right].
\end{align}
Thus, the regret of $\sB$ is described as
\begin{align}
    \mathfrak{R}_{\sB}(T) &= \sum_{t=1}^{T} \bbE \left[ g(y_{t}, x_{h^{*}(t),t})  - g(y_{t}, x_{\sB(t),t})\right] \nn \\
    &= \sum_{t=1}^{T} \bbE \left[ g(y_{t}, x_{h^{*}(t),t}) -g(\hat{y}_{t}, x_{h^{*}(t),t}) + g(\hat{y}_{t}, x_{h^{*}(t),t})  - g(\hat{y}_{t}, x_{\sB(t),t}) + g(\hat{y}_{t}, x_{\sB(t),t})- g(y_{t}, x_{\sB(t),t}) \right] \nn \\
    & \leq \sum_{t=1}^{T} \bbE \left[ L \frac{\d}{2} \right] + \sum_{t=1}^{T} \bbE \left[ g(\hat{y}_{t}, x_{h^{*}(t),t})  - g(\hat{y}_{t}, x_{\sB(t),t}) \right] + \sum_{t=1}^{T} \bbE \left[ L \frac{\d}{2} \right] \tag{L-Lipschitz condition} \\ 
    & \leq \d LT +  \sum_{t=1}^{T} \bbE \left[ g(\hat{y}_{t}, x_{\hat{h}^{*}(t),t})  - g(\hat{y}_{t}, x_{\sA(t),t}) \right] \nn \\
    & \leq \d LT +  \mathfrak{R}_{\sA}(T) . \label{eq:reg_CCB}
\end{align}
From theorem \ref{thm:reg_UB}, we know that $\mathfrak{R}_{\sA}(T)$ is at most $O(\log T)$. Substituting this in (\ref{eq:reg_CCB}) yields the result.
\end{proof}

Note that the choice of $\d$ is dependent on the time horizon $T$. The regret upper bound is, therefore, not linear in $T$ as it might appear from theorem \ref{thm:CCB_bound}. In the following subsections, we discuss how $\d$ can be tuned to $T$ in order to obtain the desired storage and regret guarantees. Hereafter, we will use $\d_{T}$ to denote the tuned parameter $\d$.

\subsection{Known Time Horizon}
The CCB($\e,\d_{T}$) regret bound of $O(\d_{T} L T)$ is largely controlled by the parameter $\d_{T}$.  Apart from the regret, another important concern for online learning algorithms is the required storage. Since there are $O(\frac{1}{\d_{T}})$ context intervals, the storage required is $O(\frac{K}{\d_{T}})$. This manifests a peculiar storage vs regret trade-off. As $\d_{T}$ increases, its storage decreases, while the regret increases linearly. This trade-off arises due to the structure of our contextual bandit setting that knows the context-reward mapping function. If the time horizon is fixed and known in advance, $\d_{T}$ can be tuned based on the performance requirements. 

Note that the regret bounds of CCB($\e, \d$) are only possible due to joint learning for all context intervals. Multi-UCB can also be extended to handle continuous contexts by using this quantization technique. This essentially generalizes the Query-ad-clustering algorithm from \cite{lu2010lipschitz} with tunable $\d$. A similar algorithm employing adaptive quantization intervals appears in \cite{slivkins2014similarity}. Both of these algorithms cannot exploit the reward function knowledge and they thus have to perform independent exploration of arms for different context intervals. As $\d$ decreases, the number of context intervals increases which also increases the number of trials spent in exploring different arms. These algorithms achieve minimum regret bounds of $O(T^{2/3 + \e})$ for any $\e > 0$ and $O(T^{2/3} \log T)$ respectively. Furthermore a lower bound of $\Omega(T^{2/3})$ on all continuous contextual bandit algorithms is proved in \cite{lu2010lipschitz}. Since we assume the reward function knowledge, the agent can obtain sub-linear regret bounds as small as $\Theta(\log T)$ by appropriately choosing $\d_{T}$ in CCB($\e, \d$). This is a substantial improvement over the existing contextual bandit algorithms. Exploiting the knowledge of the reward function, therefore, helps in drastically reducing the regret for the continuous contextual bandits.

\subsection{Unknown Time Horizon}
Even when the time horizon is unknown, similar performance can be obtained using the so-called doubling trick from \cite{lu2010lipschitz, slivkins2014similarity}. It converts a bandit algorithm with a known time horizon into one that runs without that knowledge with essentially same performance guarantees. Let $\sB_{T}$ denote the CCB($\e,\d_{T}$) instance for a fixed horizon $T$ with $\d_{T}$ tuned for regret bounds of $O(L T^{\a})$ for $\a \in [0,1]$. The new algorithm $\sC$ runs in phases $m = 1,2,\cdots$ of $2^{m}$ trials each, such that each phase $m$ runs a new instance of $\sB_{2^{m}}$. Following recurrence relationship relates the regret of $\sC$ to that of $\sB_{T}$
\begin{equation}
    \mathfrak{R}_{\sC}(T) = \mathfrak{R}_{\sC}(T/2) + \mathfrak{R}_{\sB_{T/2}}(T/2).
\end{equation}
Hence we get a regret bound of $O(L T^{\a} \log T)$, while using the storage of same order $O(KT^{1-\a})$. Note that during the execution of $\sC$, previously learned values are discarded in each phase and a fresh instance of $\sB$ is run. Using those stored values for next phases may help in decreasing the regret further. 

The system designer can analyze the regret-storage trade-off and tune $\d_{T}$ based on the hardware specifications and performance requirements. This empowers the designer with more control over the algorithmic implementation than any of the existing contextual bandit algorithms. 

\section{Numerical Simulation Results} \label{sec:simu}
In this section, we present the results of numerical simulations for the channel selection problem from \ref{app:ch_sel}. Consider the case of $M = 4$ contexts uniformly distributed over the set $\sY = \{ 1, 2, 3, 4\}$. Let there be $K = 7$ arms whose arm-values are scaled Bernoulli random variables. For every $1\leq j \leq 7$, the arm values are distributed as $\pr \{ X_{j} = j \} = \frac{8-j}{10}$, $\pr \{ X_{j} = 0 \} = \frac{2 + j}{10}$. Since the genie knows these distributions, it figures out the optimal arms for all the contexts based on the expected reward matrix $(\th_{i,j})$ given by
\begin{equation*}
\setstackgap{L}{1.1\baselineskip}
\fixTABwidth{T}
(\th_{i,j}) =  \parenMatrixstack{ \boxed{0.7} & 0.6 & 0.5 & 0.4 & 0.3 & 0.2 & 0.1 \\
0.7 & \boxed{1.2} & 1.0 & 0.8 & 0.6 & 0.4 & 0.2 \\
0.7 & 1.2 & \boxed{1.5} & 1.2 & 0.9 & 0.6 & 0.3 \\
0.7 & 1.2 & 1.5 & \boxed{1.6} & 1.2 & 0.8 & 0.4 
},
\end{equation*}
where the components in the box are the optimal arms. Hence, $h^{*}(i) = i$ for all $i$. Figure \ref{fig:channel_selection_k7} plots the cumulative regret over the logarithm of trial index. We observe that $\frac{\mathfrak{R} (n)}{\log n}$ converges to a constant for the DCB($10^{-2}$) policy, whereas for the UCB1 policy it continues to increase rapidly, since the regret for UCB1 policy grows linearly for contextual bandits. It must be noted that the logarithmic regret of DCB($\e$) is due to the presence of $3$ non-optimal arms. Multi-UCB regret is also logarithmic in time albeit with higher constants due to independent exploration of arms for each context. 
\begin{figure}[t!]
    \centering
    \begin{subfigure}[b]{0.8\textwidth}
        \centering
        \includegraphics[width=0.95\textwidth]{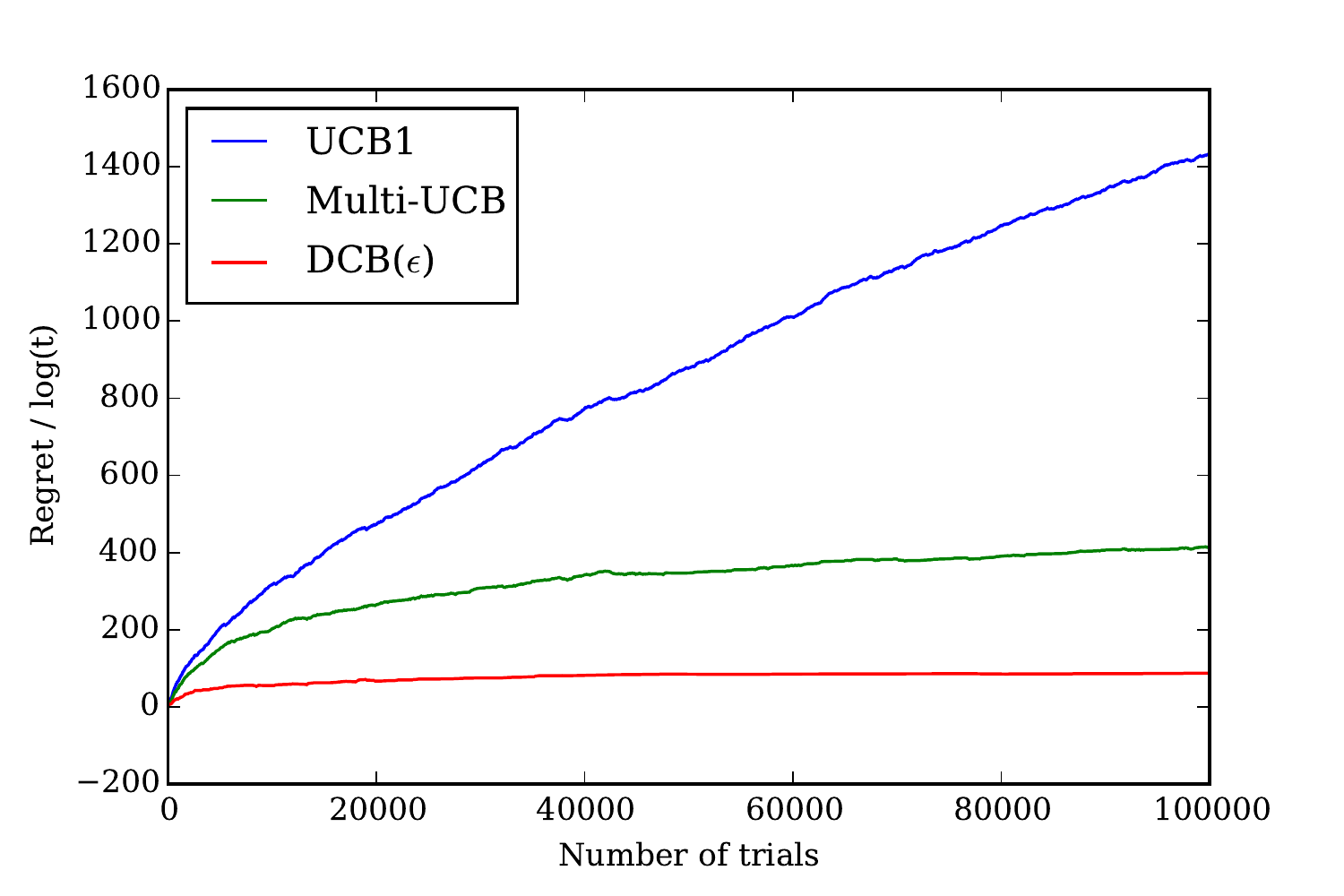}
        \caption{$K = 7$ and $\left| \overline{\sO} \right| = 3$.}
        \label{fig:channel_selection_k7}
    \end{subfigure}
%    ~
    \begin{subfigure}[b]{0.8\textwidth}
        \centering
        \includegraphics[width=0.95\textwidth]{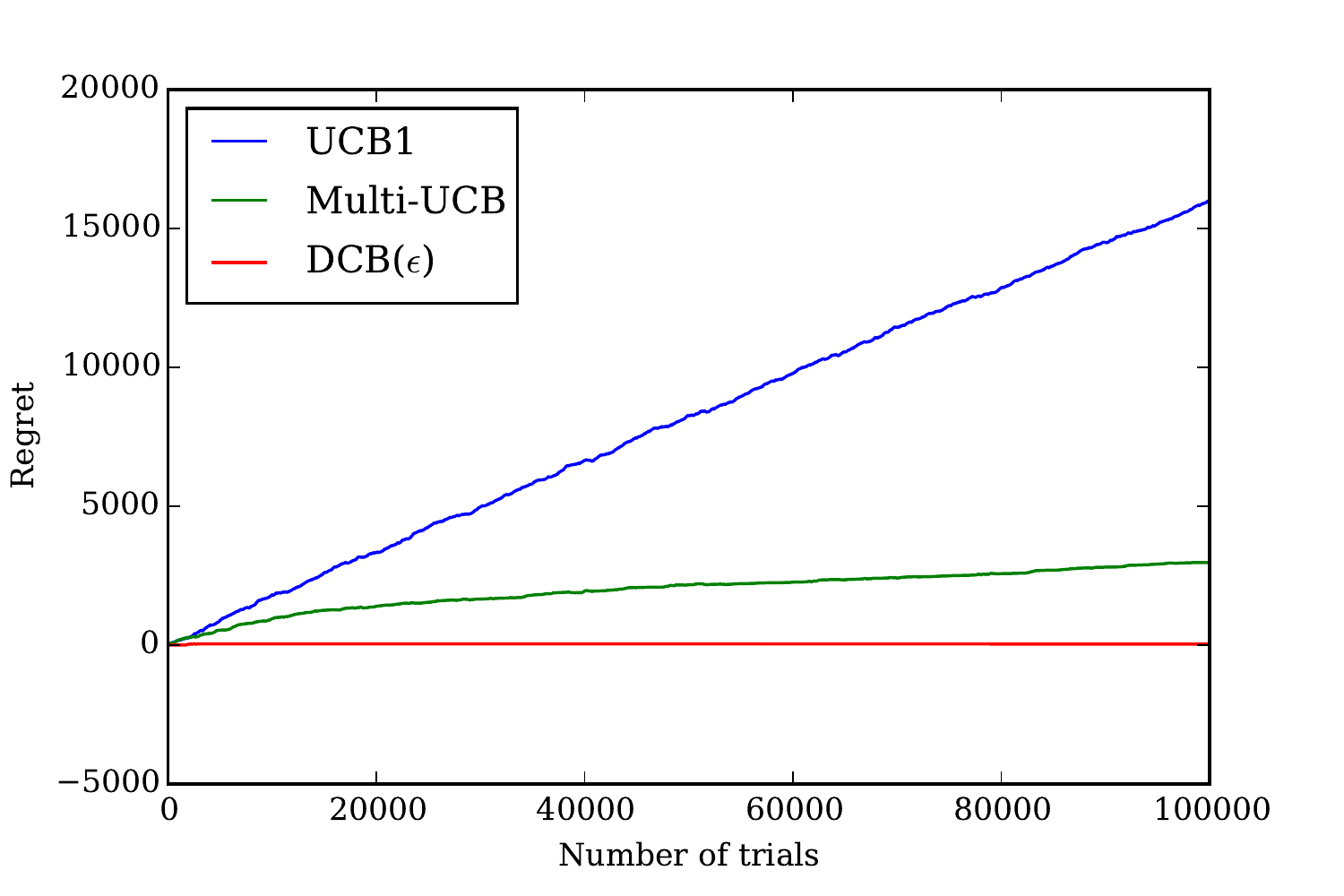}
        \caption{$K = 4$ and $\left| \overline{\sO} \right| = 0$.}
        \label{fig:channel_selection_k4}
    \end{subfigure}
    \captionsetup{justification=justified,singlelinecheck=false}
    \caption{Simulation results for the channel selection problem with $\e = 10^{-2}$.}
\end{figure}

If we reduce the number of arms to $4$, by removing the non-optimal arms $5, 6$ and $7$, then the expected reward matrix for the channel selection problem shrinks to 
\begin{equation*}
\setstackgap{L}{1.1\baselineskip}
\fixTABwidth{T}
(\th_{i,j}) =  \parenMatrixstack{ \boxed{0.7} & 0.6 & 0.5 & 0.4  \\
0.7 & \boxed{1.2} & 1.0 & 0.8 \\
0.7 & 1.2 & \boxed{1.5} & 1.2 \\
0.7 & 1.2 & 1.5 & \boxed{1.6} 
}.
\end{equation*}
Regret performance for this case is plotted in figure \ref{fig:channel_selection_k4}, which shows that the regret growth stops after some trials for DCB($10^{-2}$). It must be noted that we plot the net regret in this case and not the regret divided by $\log n$. Bounded regret is expected, since all arms are optimal and the regret due to non-optimal arms is logarithmic in time, which do not exist as $\left| \overline{\sO}\right|=0$. Since Multi-UCB is unable to exploit the reward-function knowledge, its regret still grows logarithmically in time. Such a case demonstrates the significance of our policy, as it reduced the regret drastically by jointly learning the arms-rewards for all the contexts. Table \ref{tab:regret} compares the regret for UCB1, Multi-UCB and DCB($10^{-2}$) when $T =  10^{5}$ for both channel selection examples. We see that the regret reduction using DCB($\e$) is substantial, especially when the non-optimal arm set is empty.

\begin{table}[htbp]\caption{Regret for the channel selection problem when $T = 10^{5}$}
\centering
\begin{tabular}{l | c c c}
\hline
& UCB1 & Multi-UCB & DCB($10^{-2}$) \\ \hline
$M = 4, K = 7, \left| \overline{\sO}\right| = 3$ & 17262 & 4893 & 1294 \\ \hline
$M = 4, K = 4, \left| \overline{\sO}\right| = 0$ & 15688 & 3278 & 28 \\ \hline
\end{tabular}
\label{tab:regret}
\end{table}

\begin{figure}
    \centering
    \includegraphics[width=0.8\textwidth]{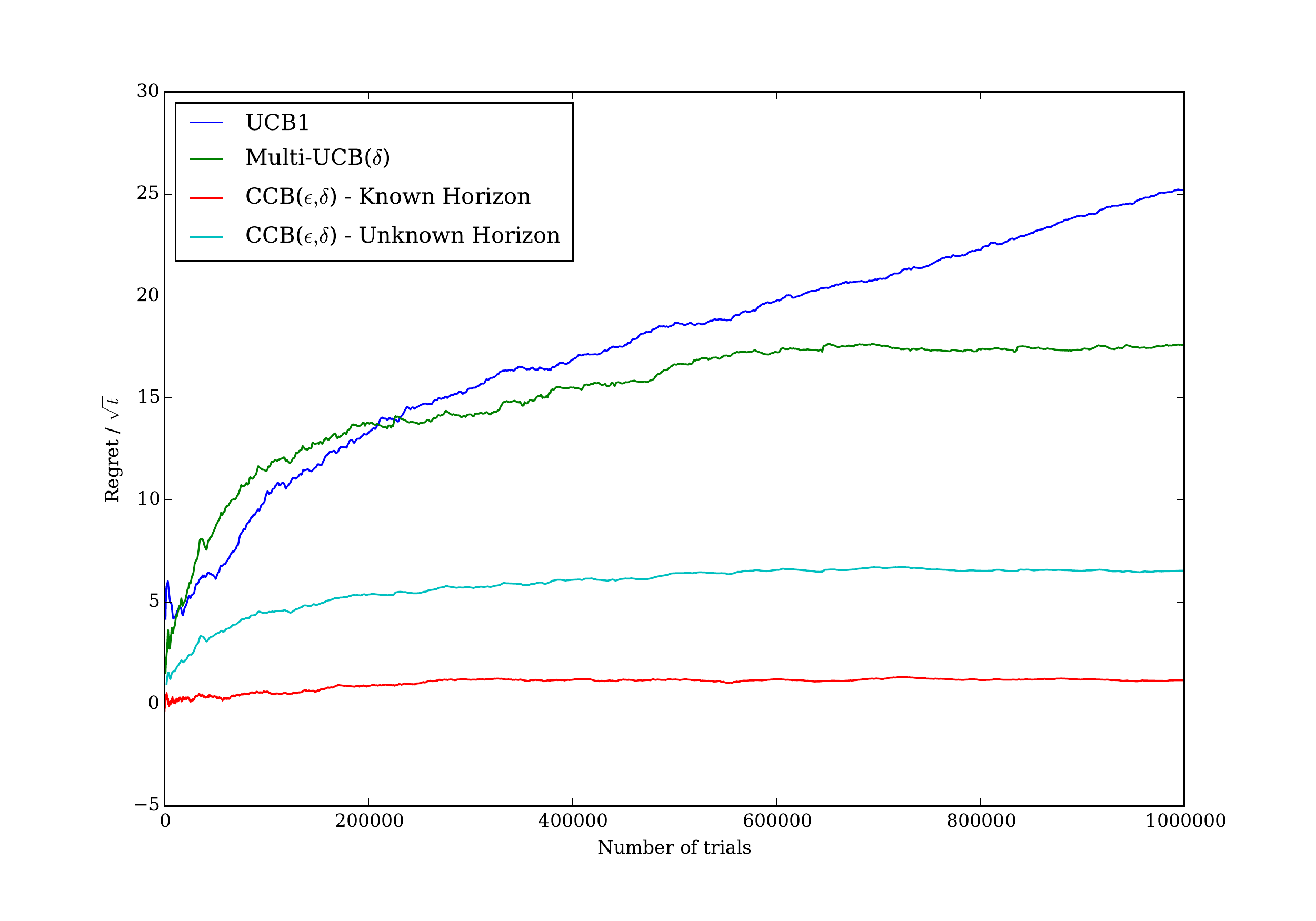}
    \caption{Simulation results for the energy harvesting communications problem with $K = 4$, $\e = 10^{-2}$ and $\d = \frac{1}{\sqrt{T}}$ for tuned algorithms.}
    \label{fig:power_alloc}
\end{figure}

\begin{table}[htbp]\caption{Regret for the power-aware channel selection problem  when $T = 10^{6}$}
\centering
\begin{tabular}{l | c | c | c}
\hline
& $\d = T^{-1/3}$ & $\d = T^{-1/2}$ & $\d = T^{-2/3}$ \\ \hline
Multi-UCB($\d)$ & 15535.8 & 17583.9 & 23117.2 \\ \hline
CCB($\e, \d$) - Unknown $T$ & 8645.7 & 6533.0 & 1476.2 \\ \hline
CCB($\e, \d$) - Known $T$ & 3010.5 & 1163.4 & 481.8 \\ \hline
UCB1 & \multicolumn{3}{c}{25201.5} \\ \hline
\end{tabular}
\label{tab:regret2}
\end{table}

We compare the performance of continuous contextual bandit algorithms for the power-aware channel selection problem in energy harvesting communications from section \ref{app:harvest}. We use the same $K = 4$ arms from the previous example. The context is assumed to be uniformly distributed in $(0,1)$. Note that arms 3 and 4 are the optimal arms in this case. We use $\d = \frac{1}{\sqrt{T}}$ with $T = 10^{6}$ trials for CCB($\e, \d$) with known time horizon. According to the theoretical results from section \ref{sec:conti}, CCB($\e, \d$) yields regret bounds of $O(\sqrt{T})$ and $O(\sqrt{T} \log T)$ in this setting with and without the knowledge of the horizon respectively. We also run a Multi-UCB instance with the same quantization intervals. Figure \ref{fig:power_alloc} plots $\frac{\mathfrak{R} (n)}{\sqrt{n}}$ for UCB1, Multi-UCB with quantized contexts, and CCB($\e, \d$) with known and unknown horizon. The curve corresponding to CCB($\e, \d$) with known horizon converges to a constant, while that with unknown horizon continues to grow slowly. Both of these algorithms outperform Multi-UCB which spends a large number of trials exploring the arms independently for each quantized context. Table \ref{tab:regret2} shows the regret at the horizon for the algorithms tuned with different values of $\d$. We notice that the regret of CCB($\e, \d$) decreases as $\d$ decreases. Even when the time horizon is unknown, CCB($\e, \d$) outperforms Multi-UCB and UCB1. The numerical regret results for Multi-UCB concur with the theoretical result from \cite{lu2010lipschitz} that the quantization intervals of size $\Theta(T^{-1/3})$ yield minimum regret. Thus we observe that reducing $\d$ does not necessarily reduce the regret of Multi-UCB as it does for CCB($\e, \d$).

\section{Conclusion} \label{sec:conclusion}
Multi-armed bandits have been previously used to model many networking applications such as channel selection and routing. In many sequential decision making problems in networks such as channel selection, power allocation and server selection, however, agent knows some side-information such as number of packets to be transmitted, transmit power available, features about the job to scheduled. Motivated by these applications, we have considered stochastic contextual bandit problems in this paper. In our formulation, the agent also knows the reward functions, i.e. the relationship between the context and the reward.

For the case of discrete and finite context spaces, we have proposed a UCB-like algorithm called DCB($\e$). It exploits the knowledge of reward functions for updating reward information for all contexts. We proposed a novel proof technique to bound the number of non-optimal pulls of the optimal arms by constant. This helped us obtain a regret bound that grows logarithmically in time and linearly in the number of non-optimal arms. This regret is shown to be order optimal by a natural extension of the lower bound result for standard multi-armed bandits. This regret performance is an improvement over bandit algorithms unaware of reward functions where regret grows linearly with the number of arms. For the proposed DCB($\e$) policy, the non-optimal pulls of the optimal arms are shown to be bounded for $\e >0$. It remains to be seen if such a guarantee can be provided for $\e = 0$. While proving the regret results for DCB($\e, \d$), we also proved a high probability bound on the number of optimal pulls by UCB1 in the standard MAB setting. This result could independently have potential applications in other bandit problems.

Further contributions involve extending DCB($\e$) to continuous context spaces. We proposed an algorithm called CCB($\e, \d$) for Lipschitz reward functions that uses DCB($\e$) as a subroutine. Regret analysis of CCB($\e, \d$) uncovered an interesting regret vs storage trade-off parameterized by $\d$ where the regret can be reduced by using larger storage. System designers can obtain sub-linear regrets by tuning $\d$ based on the time horizon and storage requirements. Even when the time horizon is unknown, similar performance is guaranteed by the proposed epoch-based implementation of CCB($\e, \d$). The joint learning of arm-rewards in CCB($\e, \d$) yields regret bounds that are unachievable by any of the existing contextual bandit algorithms for continuous context spaces. 

In the current setting, we assumed no queuing of data packets or the harvested energy. When there exist data queues or batteries at the transmitter, the agent can decide to send some of those data-bits or use some of that energy in the current slot and potentially store the rest for later slots. Such a setting with this additional layer of decision making is a non-trivial extension that warrants further investigation. 

\bibliography{MAB_list}{}
\bibliographystyle{ieeetr}

\appendices

\section{Sum of Bounded Random Variables} \label{apx:inequalities}
We use following version of Hoeffding's inequality from \cite{hoeffding1963}.

\begin{lemma}[Hoeffding's Inequality] \label{fact:hoeffding}
Let $Y_{1}, ..., Y_{n}$ be i.i.d. random variables with mean $\m$ and range $[0,1]$. Let $S_{n} = \sum\limits_{t=1}^{n} Y_{t}$. Then for all $\a \geq 0$
\begin{align}
\pr \{ S_{n} \geq n \m + \a \} & \leq \me^{-2\a^{2}/n}  \nn \\
\pr \{ S_{n} \leq n \m - \a \} & \leq \me^{-2\a^{2}/n} \nn.
\end{align}
\end{lemma}

\section{Proof of Theorem \ref{thm:non_opt_T_UB}: Bound for Non-Optimal Arms} \label{apx:non_opt_arm}
For any $j \in \overline{\sO}$, we write 
\begin{align}
T_{j}(n) &= 1 + \sum_{t=K+1}^{n} \bbI \{ \sA(t) = j \} \nn \\
& \leq l + \sum_{t=K+1}^{n} \bbI \{ \sA(t) = j , T_{j}(t-1) \geq l\},
\end{align}
where $\bbI(x)$ is the indicator function defined to be $1$ when the predicate $x$ is true and $0$ otherwise, and $l$ is an arbitrary positive integer. Let $C_{t,s} = \sqrt{\frac{(2+\e)\ln t}{s}}$. Let $\hat{\th}_{i,j,s}$ denote $\hat{\th}_{i,j}$ when $j$-th arm has been pulled $s$ times. We use a subscript $^{*}$ to refer to optimal arms' statistics. For example we write $\hat{\th}_{i,s}^{*}$ and $T_{i}^{*}(t-1)$ to respectively denote $\hat{\th}_{i,j,s}$ and $T_{j}(t-1)$ for the optimal arm: $j = h^{*}(\bfy_{(i)})$. Our idea is to upper bound the probability of the indicator function of the event $\{ \sA(t) = j , T_{j}(t-1) \geq l \}$ and bound the number of pulls as
\begin{align}
\bbE \left[ T_{j}(n) \right] & \leq l + \sum_{t=K+1}^{n} \pr \{ \sA(t) = j , T_{j}(t-1) \geq l\} \nn \\
& = l + \sum_{t=K+1}^{n} \sum_{i = 1}^{M} \pr \{ \sA(t) = j , T_{j}(t-1) \geq l, \bfy_{t} = \bfy_{(i)}\} \nn \\
& = l + \sum_{i = 1}^{M} p_{i} \sum_{t=K+1}^{n}  \pr \{ \sA(t) = j , T_{j}(t-1) \geq l \mid \bfy_{t} = \bfy_{(i)}\} \nn \\
& \leq l + \sum_{i = 1}^{M} p_{i} \sum_{t=K+1}^{n}  \pr \left\{ \hat{\th}_{i,T_{i}^{*}(t-1)}^{*} + G_{i} C_{t-1,T_{i}^{*}(t-1)} \leq \hat{\th}_{i,j,T_{j}(t-1)} + G_{i} C_{t-1,T_{j}(t-1)} , T_{j}(t-1) \geq l \right\} \nn \\
& \leq l + \sum_{i = 1}^{M} p_{i} \sum_{t=K+1}^{n}  \pr \left\{ \min\limits_{1\leq s < t} \left( \hat{\th}_{i,s}^{*} + G_{i} C_{t-1,s} \right) \leq \max\limits_{l \leq s_{j} < t} \left( \hat{\th}_{i,j,s_{j}} + G_{i} C_{t-1,s_{j}} \right) \right\} \nn \\
& \leq l + \sum_{i = 1}^{M} p_{i} \sum_{t=1}^{\infty} \sum_{s = 1}^{t-1} \sum_{s_{j}=l}^{t-1} \pr \left\{ \hat{\th}_{i,s}^{*} + G_{i} C_{t,s} \leq \hat{\th}_{i,j,s_{j}} + G_{i} C_{t,s_{j}} \right\}. \label{eq:thm1_part1}
\end{align}
Observing that $\hat{\th}_{i,s}^{*} + G_{i} C_{t,s} \leq \hat{\th}_{i,j,s_{j}} + G_{i} C_{t,s_{j}}$ cannot hold, unless at least one of the following conditions hold
\begin{equation}
\hat{\th}_{i,s}^{*} \leq   \th_{i}^{*} - G_{i} C_{t,s}, \label{eq:thm1_cond1} 
\end{equation}
\begin{equation}
\hat{\th}_{i,j,s_{j}} \geq  \th_{i,j}  + G_{i} C_{t,s_{j}}, \label{eq:thm1_cond2} 
\end{equation}
\begin{equation}
 \th_{i}^{*}  \leq  \th_{i,j} + 2 G_{i} C_{t,s_{j}}. \label{eq:thm1_cond3} 
\end{equation}
Using Hoeffding's inequality (see appendix \ref{apx:inequalities}) on (\ref{eq:thm1_cond1}) and (\ref{eq:thm1_cond2}), we get
\begin{equation}
\pr \{ \hat{\th}_{i,s}^{*} \leq   \th_{i}^{*} - G_{i} C_{t,s} \} \leq \me^{-2(2+\e) \ln t} = t^{-2(2+\e)}, \nn
\end{equation}
\begin{equation}
\pr \{\hat{\th}_{i,j,s_{j}} \geq  \th_{i,j}  + G_{i} C_{t,s_{j}} \} \leq \me^{-2(2+\e) \ln t} = t^{-2(2+\e)}. \nn
\end{equation}

For $s_{j} \geq l = \left\lceil \frac{4(2+\e) \ln n}{\min\limits_{1\leq i \leq M} \left( \D_{j}^{(i)} \right)^{2}} \right\rceil $, we get
\begin{align}
\th_{i,j^{*}_{i}} - \th_{i,j} - 2 G_{i} C_{t,s_{j}} & = \th_{i}^{*} - \th_{i,j} - 2 G_{i} \sqrt{\frac{(2+\e)\ln t}{s_{j}}} \nn \\
&\geq \D_{j}^{(i)} - \min\limits_{1\leq i \leq M} \D_{j}^{(i)} \nn \\
&\geq 0. \nn
\end{align}
Hence, the condition (\ref{eq:thm1_part1}) reduces to 
\begin{align}
\bbE \left[ T_{j}(n) \right] & \leq \left\lceil \frac{4(2+\e) \ln n}{\min\limits_{1\leq i \leq M} \left( \D_{j}^{(i)} \right)^{2}} \right\rceil + \sum_{i = 1}^{M} p_{i} \sum_{t=1}^{\infty} \sum_{s = 1}^{t-1} \sum_{s_{j}=l}^{t-1} \left( \pr \left\{ \hat{\th}_{i,s}^{*} \leq   \th_{i}^{*} - G_{i} C_{t,s} \right\} +  \pr \left\{ \hat{\th}_{i,j,s_{j}} \geq  \th_{i,j}  + G_{i} C_{t,s_{j}} \right\} \right) \nn \\
& \leq \left\lceil \frac{4(2+\e) \ln n}{\min\limits_{1\leq i \leq M} \left( \D_{j}^{(i)} \right)^{2}} \right\rceil + \sum_{i = 1}^{M} p_{i} \sum_{t=1}^{\infty} \sum_{s = 1}^{t} \sum_{s_{j}=l}^{t}  2 t^{-2(2+\e)} \nn \\
& \leq \left\lceil \frac{4(2+\e) \ln n}{\min\limits_{1\leq i \leq M} \left( \D_{j}^{(i)} \right)^{2}} \right\rceil + 2 \sum_{i = 1}^{M} p_{i} \sum_{t=1}^{\infty} t^{-2(1+\e)} \nn \\
& \leq \left\lceil \frac{4(2+\e) \ln n}{\min\limits_{1\leq i \leq M} \left( \D_{j}^{(i)} \right)^{2}} \right\rceil + \frac{\pi^{2}}{3} \sum_{i = 1}^{M} p_{i} \nn \\
& \leq \frac{4(2+\e) \ln n}{\min\limits_{1\leq i \leq M} \left( \D_{j}^{(i)} \right)^{2}} + 1+ \frac{\pi^{2}}{3}, \nn
\end{align}
which concludes the proof.
$\QED$

\section{Proof of Lemma \ref{lem:high_prob_ucb}: High probability bound for UCB1($\e$)} \label{apx:high_prob_ucb}
Let $\bm{\pi}$ denote the instance of UCB1($\e$) in the standard MAB problem and $j^{*}$ the index of optimal arm. UCB1($\e$) stores the empirical means of arm-rewards. We use $\hat{X}_{j,s}$ and $\hat{X}_{s}^{*}$ to respectively denote the averages of arm rewards for $j$-th arm and the optimal arm when they have been pulled for $s$ trials. 

If the optimal arm gets pulled less than $\frac{n}{K}$ times during the first $n$ trials, then according to the pigeonhole principle there must exist an arm that gets pulled more than $\frac{n}{K}$ times. 
\begin{align}
\pr \left\{ T^{*}(n) < \frac{n}{K} \right\} &\leq \pr \left\{\exists j \in \sS \backslash \{j^{*}\} : T_{j}(n) > \frac{n}{K} \right\} \nn \\
& \leq \sum_{j\in \sS \backslash \{j^{*}\}} \pr \left\{ T_{j}(n) > \frac{n}{K} \right\}. \label{eq:lem_bound1}
\end{align}
For the $j$-th arm, if $T_{j}(n) \geq \left\lfloor \frac{n}{K} \right\rfloor + 1$, then at some $\left\lfloor \frac{n}{K} \right\rfloor + 1 \leq t \leq n$ it must have been pulled for the ($\left\lfloor \frac{n}{K} \right\rfloor + 1$)-th time. We track this trial-index $t$ and bound the probabilities for non-optimal arms.
\begin{align}
\pr \left\{ T_{j}(n) > \frac{n}{K} \right\} &\leq \sum_{t = \left\lfloor \frac{n}{K}\right\rfloor +1}^{n} \pr \left\{ \bm{\pi}(t) = j, T_{j}(t-1) = \left\lfloor \frac{n}{K}\right\rfloor \right\} \nn \\
&\leq \sum_{t = \left\lfloor \frac{n}{K}\right\rfloor +1}^{n} \pr \left\{ \hat{X}^{*}_{t^{*}(t-1)} + C_{t-1, T^{*}(t-1)} \leq \hat{X}_{j,\left\lfloor \frac{n}{K}\right\rfloor} + C_{t-1, \left\lfloor \frac{n}{K}\right\rfloor}  \right\} \nn \\
&\leq \sum_{t = \left\lfloor \frac{n}{K}\right\rfloor +1}^{n} \pr \left\{\min_{1 \leq s < t} \hat{X}^{*}_{s} + C_{t-1, s} \leq \hat{X}_{j,\left\lfloor \frac{n}{K}\right\rfloor} + C_{t-1, \left\lfloor \frac{n}{K}\right\rfloor}  \right\} \nn \\
&\leq \sum_{t = \left\lfloor \frac{n}{K}\right\rfloor +1}^{n} \sum_{s =1}^{t-1} \pr \left\{\hat{X}^{*}_{s} + C_{t-1, s} \leq \hat{X}_{j,\left\lfloor \frac{n}{K}\right\rfloor} + C_{t-1, \left\lfloor \frac{n}{K}\right\rfloor}  \right\} . \label{eq:lem_1_arm}
\end{align}
Note that $\hat{X}^{*}_{s} + C_{t-1, s} \leq \hat{X}_{j,\left\lfloor \frac{n}{K}\right\rfloor} + C_{t-1, \left\lfloor \frac{n}{K}\right\rfloor}$ implies that at least one of the following must hold
\begin{align}
\hat{X}^{*}_{s} &\leq \m^{*} - C_{t,s} \label{eq:lem_cond1} \\
\hat{X}_{j,\left\lfloor \frac{n}{K}\right\rfloor} &\geq \m_{j} - C_{t,\left\lfloor \frac{n}{K}\right\rfloor} \label{eq:lem_cond2} \\
\m^{*} &< \m_{j} + 2C_{t,\left\lfloor \frac{n}{K}\right\rfloor} \label{eq:lem_cond3}
\end{align}
Using Hoeffding's inequality we bound the probability of events (\ref{eq:lem_cond1}) and (\ref{eq:lem_cond2}) as
\begin{equation}
\pr \left\{ \hat{X}^{*}_{s} \leq \m^{*} - C_{t,s} \right\} \leq \me^{-2(2+\e)\ln t} = t^{-2(2+\e)}, \nn
\end{equation}
\begin{equation}
\pr \left\{ \hat{X}_{j,\left\lfloor \frac{n}{K}\right\rfloor} \geq \m_{j} - C_{t,\left\lfloor \frac{n}{K}\right\rfloor} \right\} \leq \me^{-2(2+\e)\ln t} = t^{-2(2+\e)}. \nn
\end{equation}
As stated in (\ref{eq:cond_ucb}) for $\left\lfloor \frac{n}{K}\right\rfloor > \frac{4(2+\e) \ln n}{\min\limits_{j \neq j^{*}} (\m^{*} - \m_j)^{2}}$, 
\begin{align}
\m^{*} - \m_{j} - 2C_{t,\left\lfloor \frac{n}{K}\right\rfloor} &= \m^{*} - \m_{j} - 2\sqrt{\frac{(2+\e)\ln t}{\left\lfloor \frac{n}{K}\right\rfloor}} \nn \\
&> \m^{*} - \m_{j} - 2\sqrt{\frac{(2+\e)\ln n}{\left\lfloor \frac{n}{K}\right\rfloor}} \nn \\
&> \m^{*} - \m_{j} - \min\limits_{1\leq j \leq K} (\m^{*} - \m_j) \nn \\
&\geq 0.
\end{align}
Thus, condition (\ref{eq:lem_cond3}) is always false when $\left\lfloor \frac{n}{K}\right\rfloor > \frac{4(2+\e) \ln n}{\min\limits_{j \neq j^{*}} (\m^{*} - \m_j)^{2}}$.

Using union bound on (\ref{eq:lem_1_arm}) for every non-optimal arm, we get
\begin{align}
\pr \left\{ T_{j}(n) > \frac{n}{K} \right\} &\leq \sum_{t = \left\lfloor \frac{n}{K}\right\rfloor +1}^{n} \sum_{s =1}^{t-1} \left(\pr \left\{ \hat{X}^{*}_{s} \leq \m^{*} - C_{t,s} \right\}+  \pr \left\{ \hat{X}_{j,\left\lfloor \frac{n}{K}\right\rfloor} \geq \m_{j} - C_{t,\left\lfloor \frac{n}{K}\right\rfloor} \right\} \right) \nn \\
&\leq  \sum_{t = \left\lfloor \frac{n}{K}\right\rfloor +1}^{n} \sum_{s =1}^{t-1} 2 t^{-2(2+\e)} \nn \\
&\leq  2\sum_{t = \left\lfloor \frac{n}{K}\right\rfloor +1}^{n}  t^{-3-2\e} \nn \\
&<  2\sum_{t = \left\lfloor \frac{n}{K}\right\rfloor +1}^{n}  \left(\frac{n}{K}\right)^{-3-2\e} \nn \\
&<  2n \left(\frac{n}{K}\right)^{-3-2\e} \nn \\
& = \frac{2 K^{3 + 2\e}}{n^{2+2\e}}. \label{eq:lem_1_arm_b}
\end{align}
Substituting this in (\ref{eq:lem_bound1}), for all $n$ satisfying condition (\ref{eq:cond_ucb}), we get
\begin{align}
\pr \left\{ T^{*}(n) < \frac{n}{K} \right\} &< \sum_{j \in \sS \backslash \{ j^{*}\} } \frac{2 K^{3 + 2\e}}{n^{2+2\e}} \nn \\
& < \frac{2K^{4+2\e}}{n^{2+2\e}}. \nn
\end{align}
 $\QED$

\section{Proof of Theorem \ref{thm:opt_T_UB}: Bound for Optimal Arms} \label{apx:opt_arm_bound}
Fix $j \in \sO$. Let $\sE_{j,t}$ denote an event where $j$-th arm is pulled non-optimally at $t$-th trial. Total number of non-optimal pulls can, therefore, be written as $T_{j}^{N}(n) = \sum\limits_{t=1}^{n} \bbI \{ \sE_{j,t} \}$. Let $\sE_{j,t}^{1}$ denote an event of at least one of the contexts $\bfy_{(i)} \in \sY_{j}$ not having occurred even half the number of its expected occurrences till the $t$-th trial. Additionally, let $\sE_{j,t}^{2}$ be the event of $j$-arm not having been pulled at least $\frac{1}{K}$ fraction of such pulls by the optimal hypothesis till the $t$-th trial. In terms of these events, we write the following bound on expected number of non-optimal pulls
\begin{align}
    \bbE [T_{j}^{N}(n)] & = \sum\limits_{t=1}^{n} \pr \{ \sE_{j,t} \} \nn \\
    & \leq \sum\limits_{t=1}^{n} \left( \pr \{ \sE_{j,t}^{1} \} + \pr \{ \sE_{j,t}^{2} \cap \overline{\sE_{j,t}^{1}} \} + \pr \{ \sE_{j,t} \cap \overline{\sE_{j,t}^{1}} \cap \overline{\sE_{j,t}^{2}} \} \right) \label{eq:non_opt_pulls1}.
\end{align}

\paragraph{Under-realized contexts} Let $N_{i}(n)$ denote the number occurrences of $\bfy_{(i)}$ till the $n$-th trial. Thus, $\sE_{j,t}^{1}$ corresponds to the event where $N_{i}(t) \leq p_i \frac{n}{2} $ for at least one context $\bfy_{(i)} \in \sY_{j}$. 

In terms of indicators, $N_{i}(n) = \sum\limits_{t=1}^{n} \bbI \{ \bfy_{t} = \bfy_{(i)}\}$ and $\bbE[N_{i}(n)] = p_i n$. These indicator random variables are i.i.d. and thus, we use Hoeffding's inequality (appendix \ref{apx:inequalities}) with $\a = p_i n/2$ to get
\begin{equation}
\pr \left\{ N_{i}(n) \leq \frac{p_{i}}{2} n \right\} \leq \exp \left\{ - \frac{p_{i}^{2}}{2} n \right\}.
\end{equation}
The exponential bound is important, since it helps us bound the number of occurrences of context under-realization $N(\sE_{j}^{1})$ by a constant. Note that $N(\sE_{j}^{1}) = \sum_{n=1}^{\infty} \bbI\{ \sE_{j,n}^{1} \}$. We obtain a bound on its expectation, which also bounds the first term in equation (\ref{eq:non_opt_pulls1}), as follows
\begin{align}
\bbE \left[ N(\sE_{j}^{1}) \right] & = \sum_{n=1}^{\infty} \pr \{ \sE_{j,n}^{1} \} \nn \\
& \leq \sum_{n=1}^{\infty} \sum\limits_{\bfy_{(i)} \in \sY_{j}} \exp \left\{ - \frac{p_{i}^{2}}{2} n \right\} \tag{Union bound} \\
& \leq \sum\limits_{\bfy_{(i)} \in \sY_{j}} \frac{\exp \left\{ - \frac{p_{i}^{2}}{2} \right\}}{1 - \exp \left\{ - \frac{p_{i}^{2}}{2} \right\}} \nn \\
& \leq \sum\limits_{\bfy_{(i)} \in \sY_{j}} \frac{2}{ p_{i}^{2} }. \label{eq:bound_e1}
\end{align}

\paragraph{Under-exploited arm} We assume that no context $\bfy_{(i)} \in \sY_{j}$ is under-realized and yet $j$-th arm is not pulled in enough number of trials. For these contexts $j$-th arm is optimal and we don't end up pulling it enough nevertheless. We define an event $\sE_{j}^{2}$ corresponding to the existence of a context $\bfy_{(i)} \in \sY_{j}$ for which the arm is under-exploited. We now bound the probability that for some context $\bfy_{(i)} \in \sY_{j}$, $T_{j,i}^{O}(n) < \frac{N_{i}(n)}{K} n$, where $T_{j,i}^{O}(n)$ denotes the number of pulls of $j$-th arm for context $\bfy_{(i)}$. Here, we use a high probability bound (see lemma \ref{lem:high_prob_ucb}) on the optimal arm for UCB1($\e$) which we prove in appendix \ref{apx:high_prob_ucb}. 

We upper bound the second term from equation (\ref{eq:non_opt_pulls1}) as follows:
\begin{align}
\bbE \{ N(\sE_{j}^{2} \mid \overline{\sE_{j}^{1}}) \} & = \sum_{n=1}^{\infty} \pr \left\{ \sE_{j,n}^{2} \cap \overline{\sE_{j,n}^{1}} \right\} \nn \\
&\overset{(a)}{\leq} n_{o} + \sum_{n =  n_{o}}^{\infty} \sum_{\bfy_{(i)} \in \sY_{i}} \pr \left\{ T_{j,i}^{O}(n) < \frac{N_{i}(n)}{K} , N_{i}(n) \geq \frac{p_{i}n}{2} \right\} \nn \\
&\leq n_{o} + \sum_{n = n_{o}}^{\infty} \sum_{\bfy_{(i)} \in \sY_{i}} \sum_{m = p_{i}n/2}^{n} \pr \left\{ T_{j,i}^{O}(n) < \frac{N_{i}(n)}{K}, N_{i}(n) = m \right\} \nn \\
&= n_{o}+\sum_{\bfy_{(i)} \in \sY_{i}} \sum_{n = n_{o}}^{\infty} \sum_{m = p_{i}n/2}^{n} \pr \left\{ T_{j,i}^{O}(n) < \frac{N_{i}(n)}{K} \middle| N_{i}(n) = m \right\} \pr\{ N_{i}(n) = m \} \nn \\
&\leq n_{o} + \sum_{\bfy_{(i)} \in \sY_{i}} \sum_{n = n_{o}}^{\infty} \sum_{m = p_{i}n/2}^{n} \pr \left\{ T_{j,i}^{O}(n) < \frac{m}{K} \right\} \nn \\
&\overset{(b)}{\leq} n_{o} + \sum_{\bfy_{(i)} \in \sY_{i}} \sum_{n = n_{o}}^{\infty} \sum_{m = p_{i}n/2}^{n} \frac{2K^{4+2\e}}{m^{2+2\e}} \tag{Lemma \ref{lem:high_prob_ucb}} \nn \\
&\leq n_{o} + \sum_{\bfy_{(i)} \in \sY_{i}} \sum_{n = n_{o}}^{\infty} \sum_{m = p_{i}n/2}^{n} \frac{2K^{4+2\e}}{\left(p_{i}n/2 \right)^{2+2\e}} \nn \\
&\leq n_{o} + \sum_{\bfy_{(i)} \in \sY_{i}} \frac{2^{3+2\e}K^{4+2\e}}{\left(p_{i}\right)^{2+2\e}} \sum_{n = n_{o}}^{\infty} \frac{1}{n^{1+2\e}}  \nn \\
&\leq n_{o} + 2^{3+2\e}K^{4+2\e} c_{\e} \sum_{\bfy_{(i)} \in \sY_{i}} \frac{1}{\left(p_{i}\right)^{2+2\e}}   \label{eq:bound_e2},
\end{align}
where $c_{\e} = \sum_{n = 1}^{\infty} \frac{1}{n^{1+2\e}} < \infty$, since we know that the series converges for $\e > 0$. Thus, we need $\e > 0$ in order to provide theoretical performance guarantees. Note that $(a)$ holds as we bound the probabilities for the first $n_{o}$ terms by $1$, and $(b)$ holds since the pulls for one context can only decrease the exploration requirements of other contexts. For $(b)$ to hold, condition of lemma \ref{lem:high_prob_ucb} needs to be satisfied, which is $\Big\lfloor \frac{p_i n}{2 K}\Big\rfloor > \frac{4 (2+\e) \ln n}{\left(\Delta_{j}^{(i)}\right)^{2}}$ for all $n \geq n_{o}$. In order to make $n_{o}$ independent of $i$ and $j$, we choose $n_{o}$ as the minimum value of $n$ satisfying following inequality:
\begin{equation}
    \Big\lfloor \frac{p_{o} n}{2 K} \Big\rfloor > \Big\lceil \frac{4 (2+\e) \ln n}{\left(\Delta_{o}\right)^{2}} \Big\rceil,
\end{equation}
where $p_{o} = \min\limits_{1 \leq i \leq M} p_i$ and $\D_{o} = \min\limits_{\forall i,j : h^{*}(i) \neq j} \D_{j}^{(i)}$.

Note that all terms in the upper bound in (\ref{eq:bound_e2}) are constant, showing that the number of occurrences of under-exploited optimal arms is bounded.

\paragraph{Dominant confidence bounds} This case corresponds to the event where no context is under-realized and the optimal arms are also pulled sufficiently often, yet an optimal arm is pulled non-optimally. This means that $j$-th arm is pulled for a context $\bfy_{(i)}\notin \sY_{j}$ at trial $n$, while $N_{i}(n) \geq \frac{p_{i}}{2} n$ and $T_{j,i}^{O}(n) \geq \frac{N_{i}(n)}{K} n$ for all $\bfy_{(i)}\in \sY_{j}$. Let $q_{j} = \sum\limits_{\bfy_{(i)} \in \sY_{j}} p_{i}$. Note that we only need to analyze $n \geq n_{o}$, since we already upper bound the probabilities of non-optimal pulls by $1$ as shown in previous case. Thus, we have
\begin{align}
T_{j}(n) & \geq \sum_{\bfy_{(i)}\in \sY_{j}} T_{j,i}^{O}(n) \nn \\
& > \sum_{\bfy_{(i)} \in \sY_{j}} \frac{N_{i}(n)}{2} \nn \\
& >  \frac{nq_{j}}{2K}  \nn.
\end{align}
Expected number of events $N(\sE_{j} \cap \overline{\sE_{j}^{1}} \cap \overline{\sE_{j}^{2}})$ gets upper bounded as follows:
\begin{align}
    N(\sE_{j} \cap \overline{\sE_{j}^{1}} \cap \overline{\sE_{j}^{2}}) &= \sum_{n=n_{o}}^{\infty} \pr \{ \sE_{j,n} \cap \overline{\sE_{j,n}^{1}} \cap \overline{\sE_{j,n}^{2}} \} \nn \\
    &= \sum_{n=n_{o}}^{\infty} \sum_{\bfy_{(i)} \notin \sY_{j}} \pr \Big\{ \sA(n) = j, T_{j}(n) \geq \frac{n q_{j}}{2K}, \bfy_{n} = \bfy_{(i)} \Big\} \nn \\
    &\leq \sum_{n=n_{o}}^{\infty} \sum_{\bfy_{(i)} \notin \sY_{j}} p_{i} \pr \left\{ \sA(n) = j, T_{j}(n) \geq \frac{n q_{j}}{2K} \middle| \bfy_{n} = \bfy_{(i)} \right\}. 
\end{align}
Following the line of argument from the proof of theorem \ref{thm:non_opt_T_UB} in appendix \ref{apx:non_opt_arm}, we get
\begin{align}
    N(\sE_{j} \cap \overline{\sE_{j}^{1}} \cap \overline{\sE_{j}^{2}}) & \leq  2 \sum_{\bfy_{(i)} \notin \sY_{j}} p_{i} \sum_{n=n_{o}}^{\infty} n^{-2(1+\e)} \nn \\
    & \leq  \frac{\pi^{2}}{3} \sum_{\bfy_{(i)} \notin \sY_{j}} p_{i}  \nn \\
    & = \frac{\pi^{2}}{3} (1 - q_{j}). \label{eq:bound_e3}
\end{align}
Note that the arguments of theorem \ref{thm:non_opt_T_UB} hold, since we only consider trials $n \geq n_{o}$.

Combining the results form (\ref{eq:bound_e1}), (\ref{eq:bound_e2}) and (\ref{eq:bound_e3}), we get
\begin{equation}
\bbE \left[ T_{j}^{N}(n) \right] \leq n_{o} +  \left( \sum\limits_{\bfy_{(i)} \in \sY_{j}} \frac{2}{p_{i}^{2}} \right) + 2^{3+2\e}K^{4+2\e} c_{\e} \left( \sum_{\bfy_{(i)} \in \sY_{i}} \frac{1}{\left(p_{i}\right)^{2+2\e}} \right)  +  \frac{\pi^{2}}{3} (1 - q_{j}), \nn
\end{equation}
which concludes the proof. $\QED$

\end{document}